\documentclass[twoside]{article}

\usepackage[accepted]{aistats2022}
\usepackage[round]{natbib}

\usepackage{amsmath,amsfonts,bm}

\def\eqref#1{equation~\ref{#1}}

\def\1{\bm{1}}

\def\vx{{\bm{x}}}
\def\vy{{\bm{y}}}

\DeclareMathAlphabet{\mathsfit}{\encodingdefault}{\sfdefault}{m}{sl}
\SetMathAlphabet{\mathsfit}{bold}{\encodingdefault}{\sfdefault}{bx}{n}

\newcommand{\given}[1][]{\:#1\vert\:}
\DeclareMathOperator{\GP}{\mathcal{GP}}
\DeclareMathOperator{\Normal}{\mathcal{N}}
\newcommand*\diff{\mathop{}\!\mathrm{d}}

\newcommand{\E}{\mathbb{E}}

\newcommand{\R}{\mathbb{R}}

\newcommand{\KL}{D_{\mathrm{KL}}}

\DeclareMathOperator*{\argmax}{arg\,max}

\usepackage{graphicx}
\usepackage{subcaption}
\usepackage{algorithm}
\usepackage{algorithmic}
\usepackage{amsthm}
\usepackage{url}
\usepackage{hyperref}
\usepackage{xcolor}
\usepackage{colortbl}
\usepackage{booktabs}

\newtheorem{lem}{Lemma}
\newtheorem{prop}{Proposition}

\begin{document}

%
\runningtitle{Feature Collapse in Deep Kernel Learning}

%
\runningauthor{Joost van Amersfoort, Lewis Smith, Andrew Jesson, Oscar Key, Yarin Gal}

\twocolumn[

  \aistatstitle{On Feature Collapse and Deep Kernel Learning\\for Single Forward Pass Uncertainty}

  \aistatsauthor{Joost van Amersfoort\textsuperscript{1},
    Lewis Smith\textsuperscript{1},
    Andrew Jesson\textsuperscript{1},
    Oscar Key\textsuperscript{2},
    Yarin Gal\textsuperscript{1}}

  \aistatsaddress{\textsuperscript{1}OATML, University of Oxford ~~ \textsuperscript{2}University College London}
]

\begin{abstract}
  Inducing point Gaussian process approximations are often considered a gold standard in uncertainty estimation since they retain many of the properties of the exact GP and scale to large datasets.
  A major drawback is that they have difficulty scaling to high dimensional inputs.
  Deep Kernel Learning (DKL) promises a solution: a deep feature extractor transforms the inputs over which an inducing point Gaussian process is defined.
  However, DKL has been shown to provide unreliable uncertainty estimates in practice.
  We study why, and show that with no constraints, the DKL objective pushes ``far-away'' data points to be mapped to the same features as those of training-set points.
  With this insight we propose to constrain DKL's feature extractor to approximately preserve distances through a bi-Lipschitz constraint, resulting in a feature space favorable to DKL.
  We obtain a model, DUE, which demonstrates uncertainty quality outperforming previous DKL and other single forward pass uncertainty methods, while maintaining the speed and accuracy of standard neural networks.
\end{abstract}

\section{INTRODUCTION}
\label{section:introduction}
Deploying machine learning algorithms as part of automated decision making systems, such as autonomous cars or medical diagnostics, requires implementing fail-safes.
Whenever the model is presented with a novel or ambiguous input, its predictions might not be trustworthy and the input should be deferred to an expert.
A compelling way to handle this situation is to use a model that does not just achieve high accuracy, but is also able to quantify its uncertainty over predictions.
Unfortunately, neural networks extrapolate overconfidently and are not able to express their uncertainty \citep{gal2016uncertainty}.
Alternatives, such as Deep Ensembles \citep{lakshminarayanan2017simple} and MC dropout \citep{gal2016dropout}, require multiple forward passes and are therefore computationally expensive.
This leads to problems with using these models in situations that require real-time decision making, which is not feasible if many forward passes are necessary.

\begin{figure*}[t]
  \begin{subfigure}{0.25\linewidth}
    \centering
    \includegraphics[width=\linewidth]{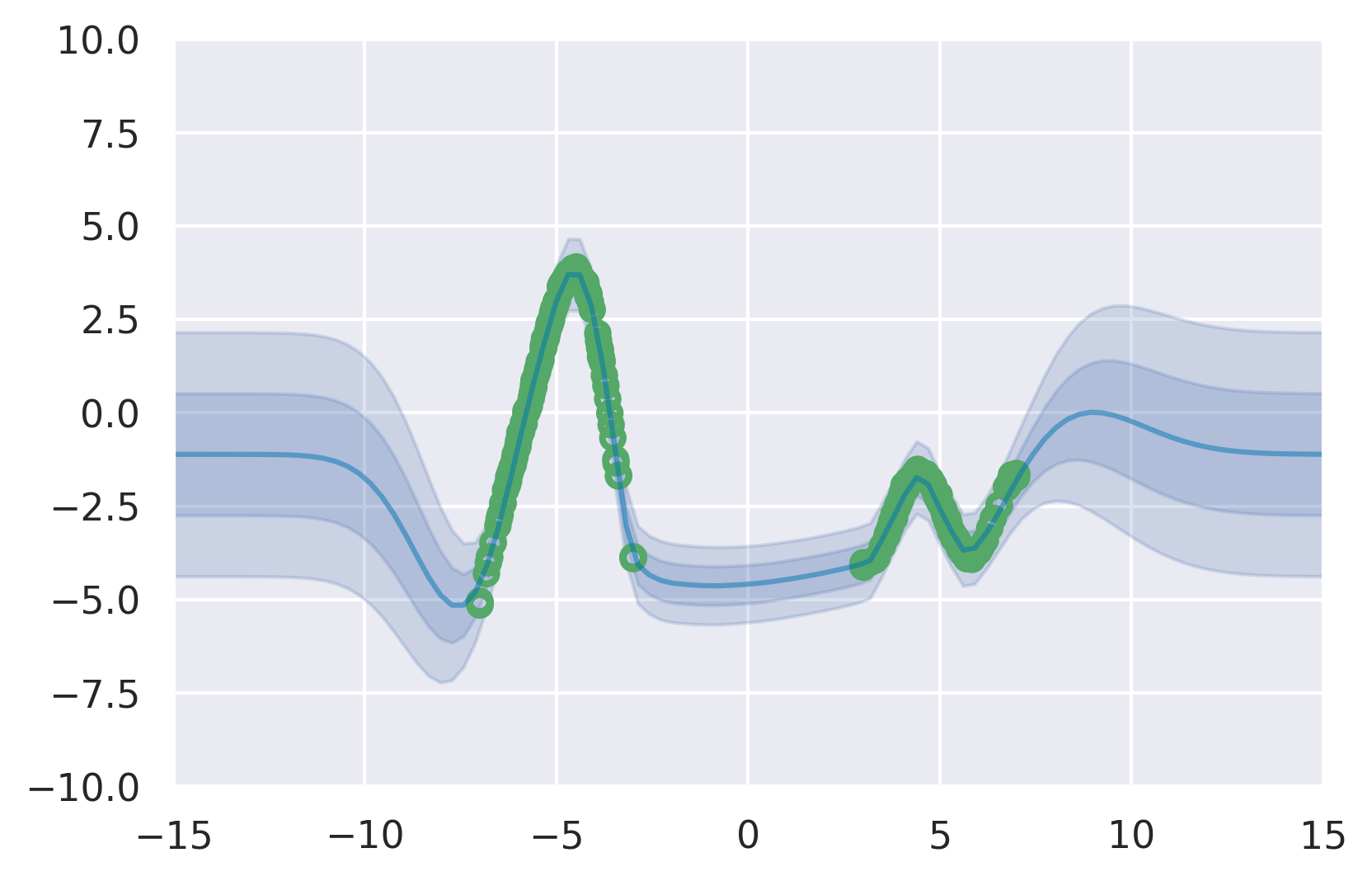}
    \caption{DUE - 1K}
    \label{fig:due_1k}
  \end{subfigure}%
  \begin{subfigure}{0.25\linewidth}
    \centering
    \includegraphics[width=\linewidth]{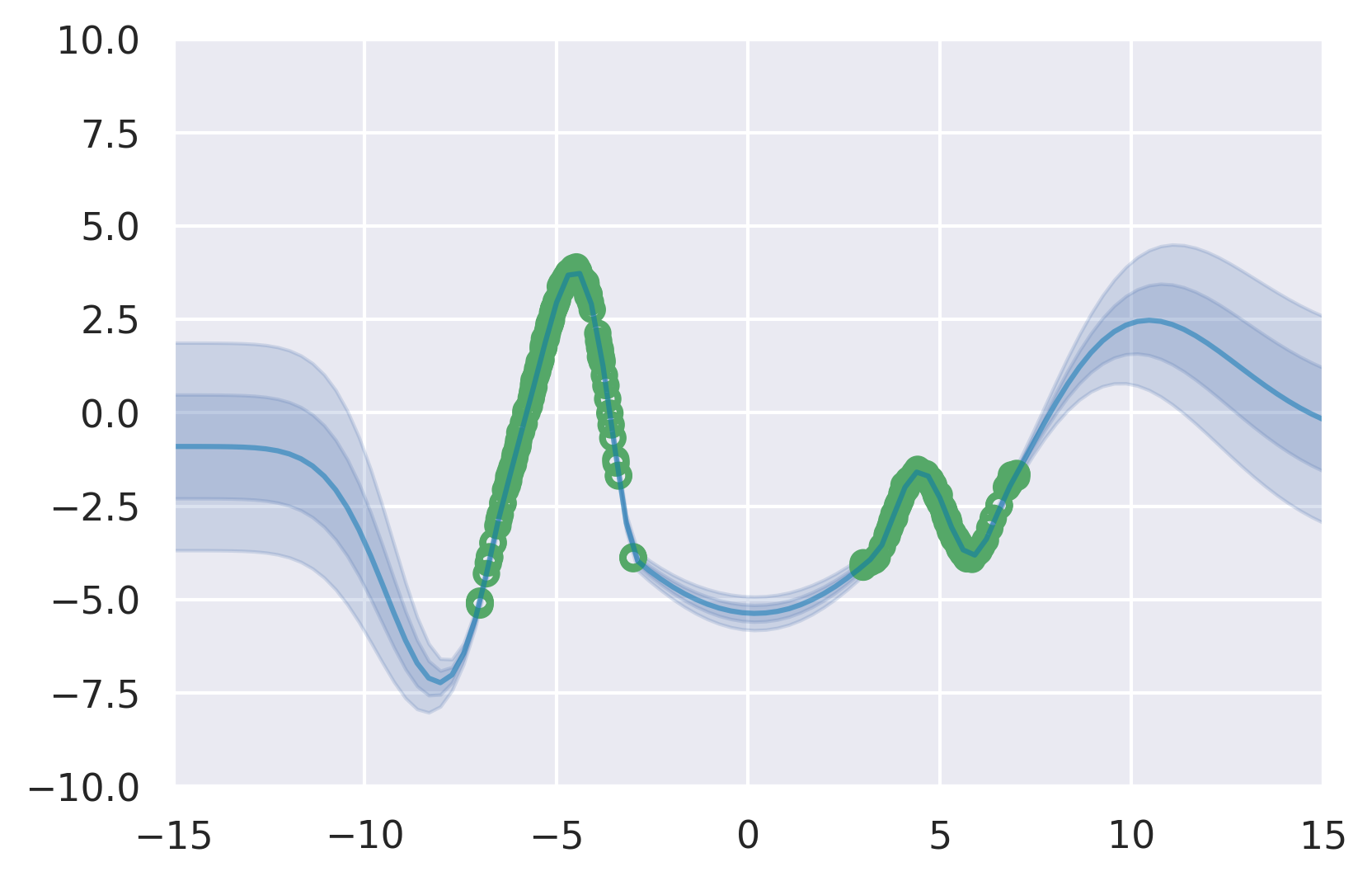}
    \caption{DUE - 1M}
    \label{fig:due_1m}
  \end{subfigure}%
  \begin{subfigure}{0.25\linewidth}
    \centering
    \includegraphics[width=\linewidth]{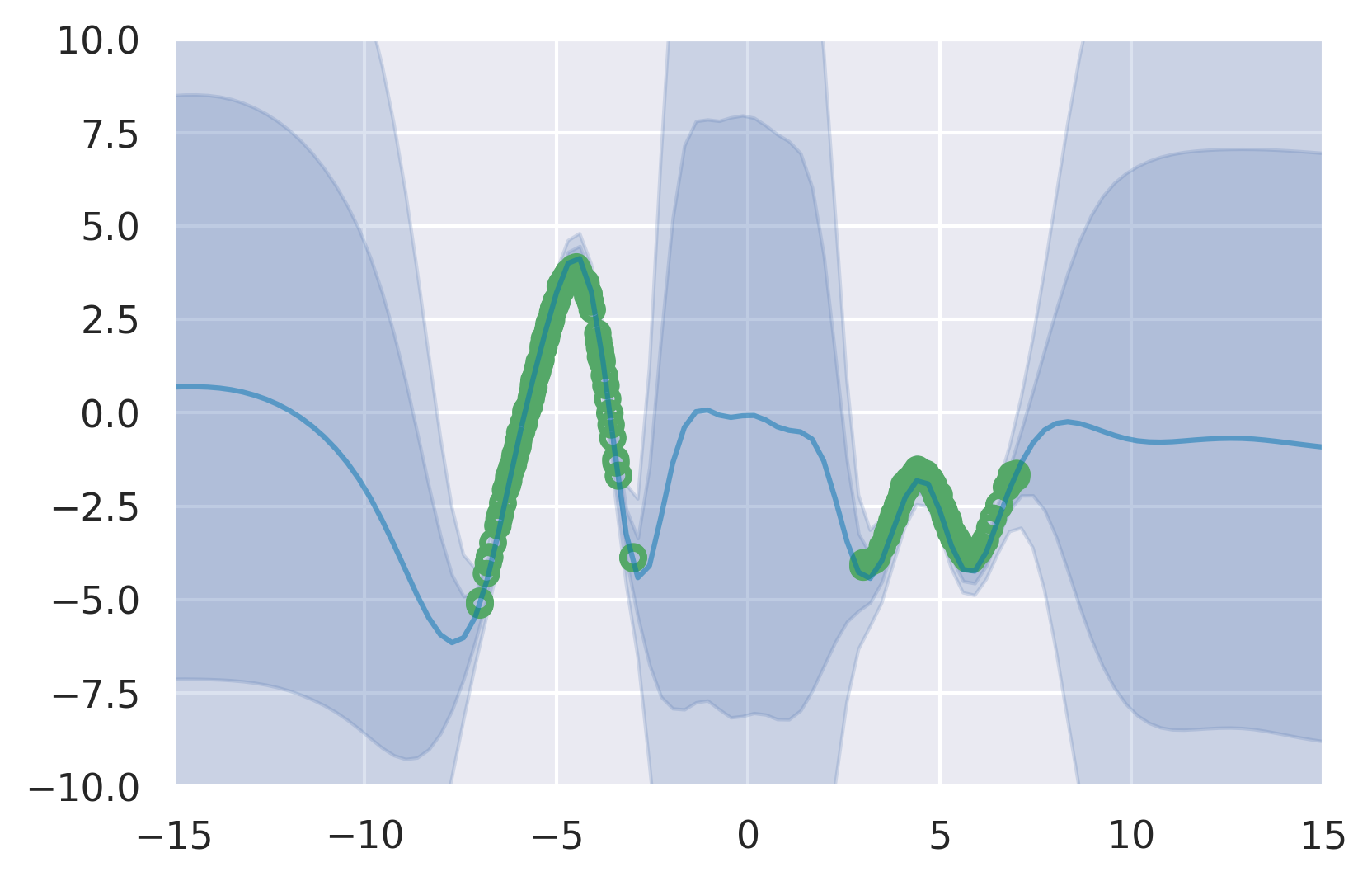}
    \caption{SNGP - 1K}
    \label{fig:sngp_1k}
  \end{subfigure}%
  \begin{subfigure}{0.25\linewidth}
    \centering
    \includegraphics[width=\linewidth]{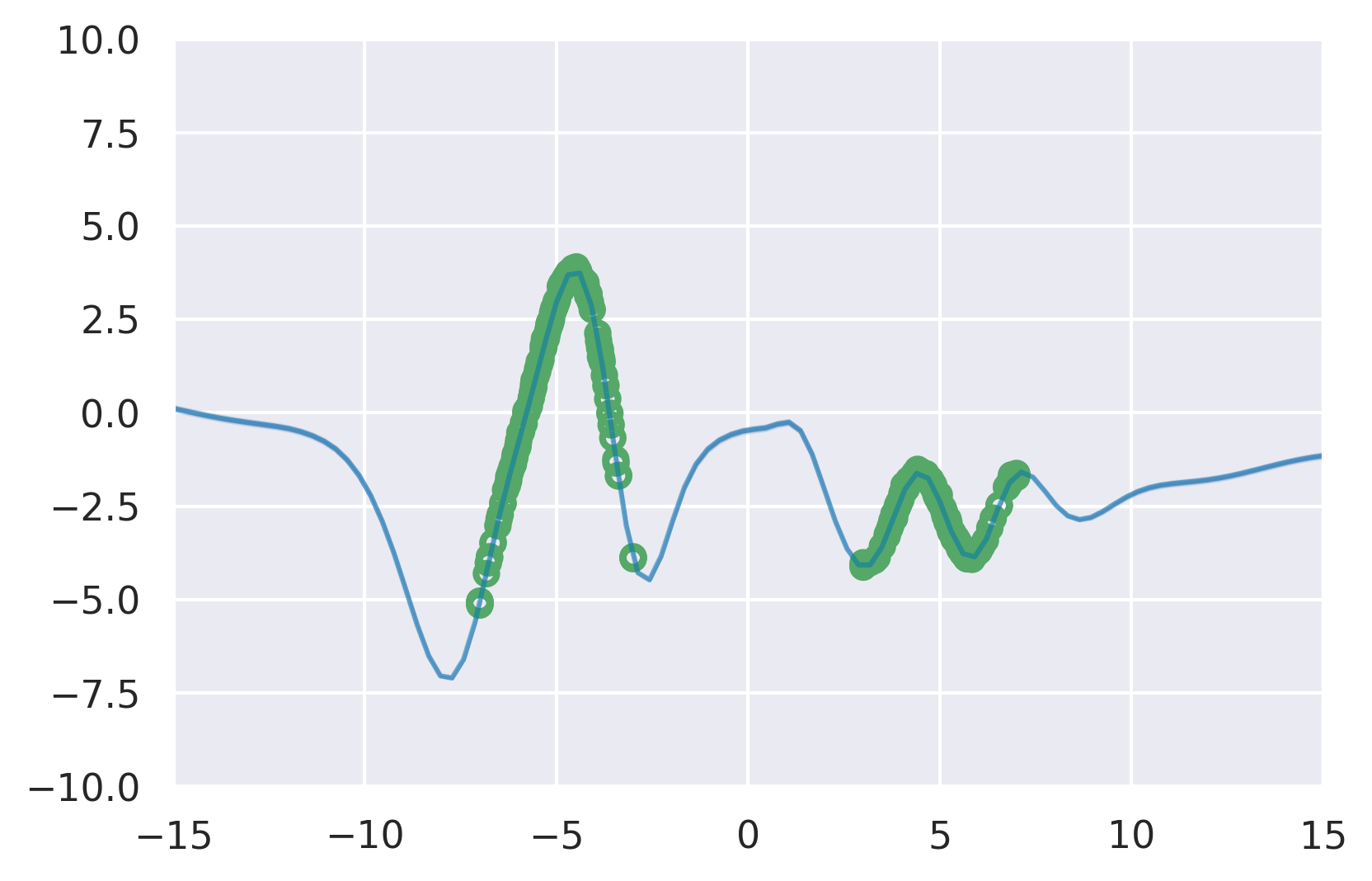}
    \caption{SNGP - 1M}
    \label{fig:sngp_1m}
  \end{subfigure}%
  \caption{
    In green 300 example training data points and in blue the prediction including uncertainty (one and two std).
    We see that DUE performs well when trained with 1 thousand (1K) datapoints and 1 million (1M) data points.
    Meanwhile, the RFF approximation in SNGP concentrates its uncertainty at 1M, and is very uncertain at 1K.
    This highlights a drawback of the parametric RFF approximation.
  }
  \label{fig:rff}
\end{figure*}
This work focuses on the problem of estimating uncertainty in deep learning in a \emph{single} forward pass.
This can in principle be achieved by using a distance-aware output function, such as a Gaussian process (GP) with a stationary kernel, which increases its uncertainty as a particular input gets further away from the training data.
However on high dimensional structured data, such as images, computing a distance in input space is less meaningful.
A practical solution is to transform the input data using a deep feature extractor and placing an approximate GP, that can scale to large datasets, on the computed features.
Two key approaches to approximate the GP are Random Fourier Features (RFF) \citep{rahimi2008random,liu2020simple} and variational inducing points, which is also known as Deep Kernel Learning (DKL) \citep{wilson2016deep, wilson2016stochastic}.
The RFF approximation, while computationally fast, sacrifices the non-parametric property of the GP and is often avoided in the GP literature (\citet{van2019sparse} and Figure \ref{fig:rff} below).
A variational inducing point approximation maintains the GP's non-parametric properties \citep{hensman2015scalable, titsias2009variational}, and can in combination with DKL obtain accuracies matching standard softmax neural networks \citep{bradshaw2017adversarial, wilson2016stochastic}.
We discuss the differences between the RFF and inducing point approximation in detail in Section \ref{subsection:approximations}.
However, maintaining good uncertainty estimates remains elusive: previous DKL works have mostly been evaluated in terms of accuracy and robustness to adversarial examples \citep{bradshaw2017adversarial, wilson2016stochastic} with recent research showing these to underperform in uncertainty estimation \citep{ober2021promises}.

We study why this happens, and propose Deterministic Uncertainty Estimation (DUE - pronounced “Dewey”), which builds on DKL and addresses its limitations.
Examining why DKL uncertainty underperforms, we find that for certain feature extractors, data points dissimilar to the training data (also called ``out-of-distribution'' or OoD data) might be mapped closed to feature representations of in-distribution points.
This is called ``feature collapse'' (Figure \ref{fig:feature_collapse}), and suggests that a constraint must be placed on the deep feature extractor.
To understand what constraints must be placed on the DKL feature extractor, we take inspiration from DUQ \citep{van2020uncertainty} and SNGP \citep{liu2020simple} which propose to use a bi-Lipschitz constraint on a feature extractor in the context of radial basis function (RBF) networks and an RFF GP.
This constraint enforces the feature representation to be sensitive to changes in the input (lower Lipschitz, avoids feature collapse) but also generalize due to smoothness (upper Lipschitz) (see Figure \ref{fig:twomoons}).

Compared to alternative single forward pass uncertainty methods, DUE has a number of advantages.
DUQ's one-vs-all loss cannot be extended to regression, while DUE is demonstrated to work well on both classification and regression.
SNGP uses the RFF approximation, which means the uncertainty concentrates in the limit of data, even ``far-away'' from the training data.
For example, in the region $-3 < x < 3$ of Figure \ref{fig:rff} we are outside of the support of the training data.
We can see for a training set size of 1K the uncertainty interval is wide for SNGP (Figure \ref{fig:sngp_1k}).
However, as we increase the training set size to 1M the uncertainty interval outside of the support of the training data becomes indistinguishable from that given for data inside the support (Figure \ref{fig:due_1m}).
In contrast to SNGP, DUE's inducing point GP preserves the exact GP's properties and has similar uncertainty outside the support of the training data when trained on small and very large datasets.
In contrast to existing DKL methods, DUE avoids feature collapse (Table \ref{table:feature_extractor}), and training DUE is substantially simplified: no pre-training is necessary and there is no computational overhead over a standard softmax model.

DUE outperforms all alternative single forward pass uncertainty methods on the CIFAR-10 vs SVHN detection task.
To evaluate DUE's uncertainty quality for regression on real world data, we use a recently introduced benchmark for predicting treatment effect and uncertainty based treatment deferral in causal models for personalized medicine \citep{jesson2020identifying}.
This benchmark combines the need for accurate predictions and uncertainty in a regression task and we show that DUE outperforms all alternative methods.
We make our code available at \url{https://github.com/y0ast/DUE}.

\begin{figure*}[t]
  \centering
  \begin{subfigure}[c]{.28\linewidth}
    \includegraphics[width=\linewidth]{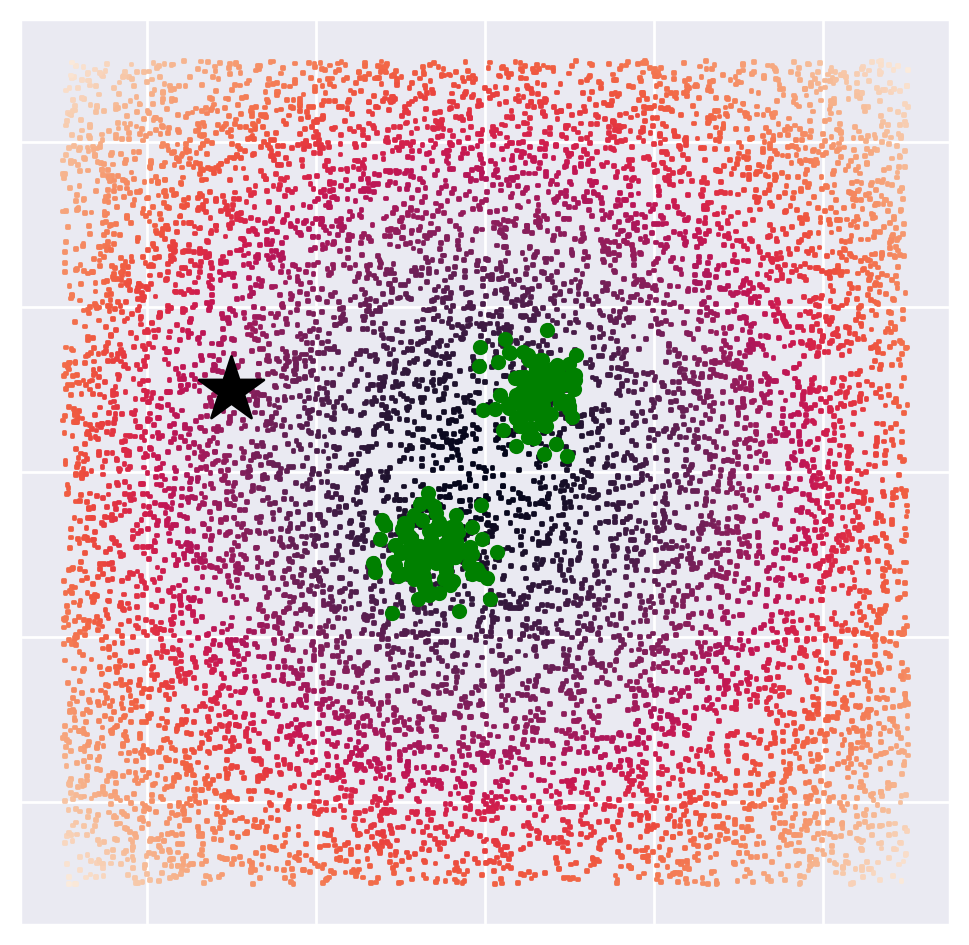}
    \caption{Input}
    \label{fig:feature_collapse_input}
  \end{subfigure}
  \begin{subfigure}[c]{.28\linewidth}
    \includegraphics[width=\linewidth]{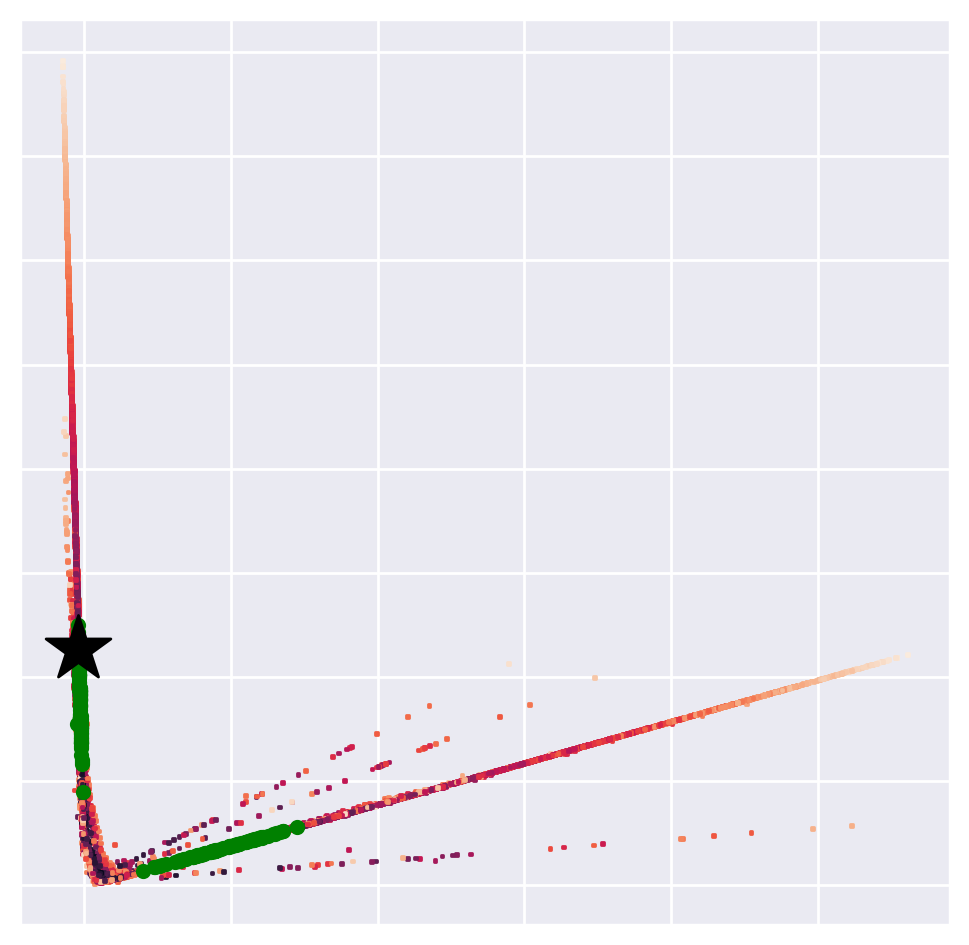}
    \caption{Trained without constraint}
    \label{fig:feature_collapse_no_constraint}
  \end{subfigure}
  \begin{subfigure}[c]{.28\linewidth}
    \includegraphics[width=\linewidth]{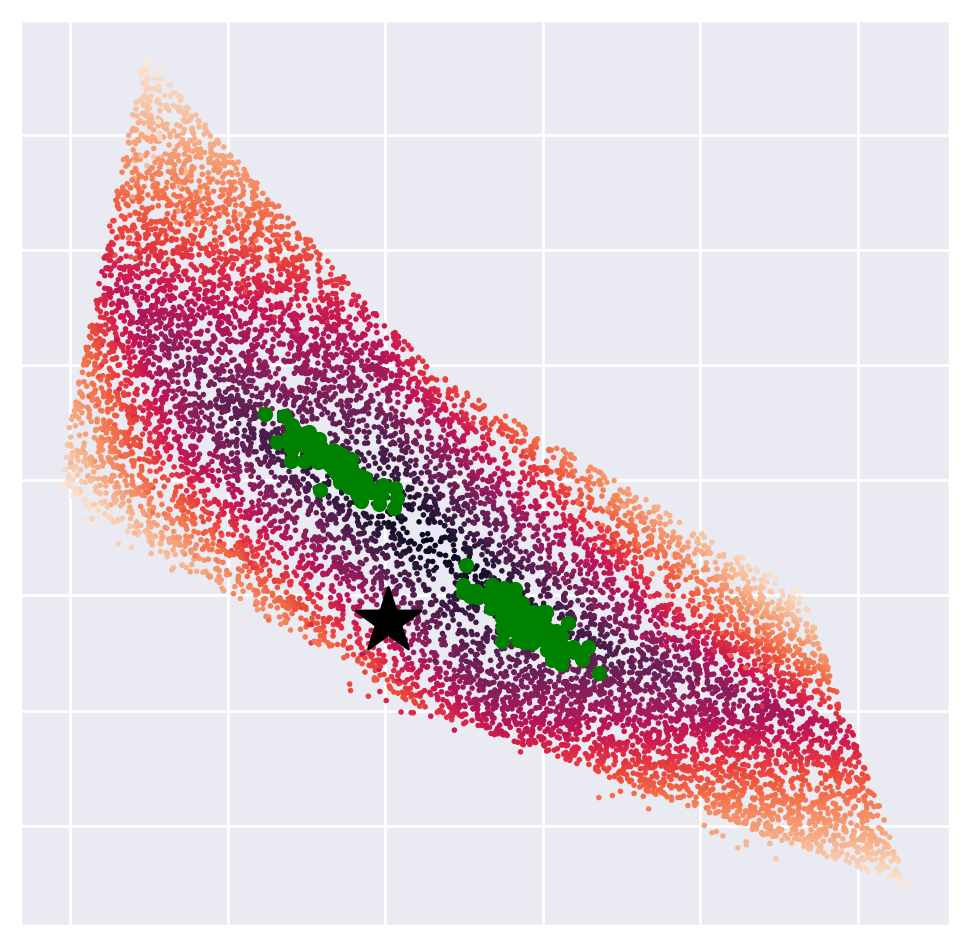}
    \caption{Trained with constraint}
    \label{fig:feature_collapse_with_constraint}
  \end{subfigure}
  \caption{
    A 2D classification task where the classes are two Gaussian blobs (drawn in green), and a grid of unrelated points (colored according to their log-probability under the data generating distribution).
    We additionally mark a specific point with a star.
    In (b), the features as computed by an unconstrained model.
    In (c), the features computed by a model with residual connections and spectral normalization.
    The objective for the unconstrained model introduces a large amount of distortion of the space, collapsing the input to a single line, making it almost impossible to use distance-sensitive measures on these features.
    In particular, the star moves from an unrelated area in input space on top of class data in feature space.
    In contrast, the constrained mapping maintains the relative distances of the other points.
  }
  \label{fig:feature_collapse}
\end{figure*}
In summary, our contributions are as follows:
\begin{enumerate}\itemsep0em
  \item A single forward pass uncertainty model that works well for classification and regression.
  \item Insight into DKL's failures and a principled solution that gives accurate uncertainty estimates.
  \item A practical DKL training setup that trains from scratch with deep feature extractors on a variety of datasets.
\end{enumerate}
\section{BACKGROUND}
\label{section:background}
In DKL, a deep feature extractor is used to transform the inputs over which an inducing point GP is defined.
DKL was originally introduced as a way to combine the expressiveness of deep neural networks with the probabilistic prediction ability of GPs \citep{hinton2008using,calandra2016manifold,wilson2016deep}.
The kernel which contains a deep feature extractor is defined as:
\begin{equation}
  \label{eq:deep_kernel}
  k_{l, \theta}(\vx_i, \vx_j) \rightarrow \bar{k}_l(f_\theta(\vx_i), f_\theta(\vx_j)) ,
\end{equation}
where $f_\theta(\cdot)$ is a deep neural network, such as a Wide ResNet (WRN) \citep{zagoruyko2016wide} up to the last linear layer, parametrized by $\theta$.
The base kernel $\bar{k}_l(\cdot, \cdot)$ can be any of the standard kernels, such as the RBF or Mat\'ern kernel. $l$ represents the hyper-parameters of the base kernel, such as the length scale and output scale.
DKL can be applied to both classification and regression problems.

Inference in an exact GP is bottlenecked by the inversion of an $n \times n$ kernel gram matrix with $n$ the number of data points, which has a time complexity that scales cubically with $n$.
In contrast, an inducing point GP is based on $m$ inducing points (with $m << n$), which act as pseudo-input-points used to approximate the full dataset.
The locations and values of the inducing points are variational parameters, and learned by maximizing a lower bound on the marginal likelihood known as the ELBO \citep{titsias2009variational,hensman2015scalable}.
This reduces the complexity of the matrix inversion from $\mathcal{O}(n^3)$ to $\mathcal{O}(m^2 n)$, thus models with fewer inducing points are faster to train.
In practice, these inducing points are placed in feature space to exploit the clustering behavior of the deep feature extractor, we visualize this behavior in the Appendix Figure \ref{fig:inducing_points}.

The two most prevalent instances of DKL are SV-DKL \citep{wilson2016stochastic} and GPDNN \citep{bradshaw2017adversarial}.
In SV-DKL, an additional restriction is placed on the inducing points to lie on a grid.
This enables faster matrix inversion algorithms, at the expense of a less flexible inducing point structure.
In GPDNN, the inducing points are not constrained, but the feature dimensionality is reduced to just 25, leading to increased risk of feature collapse and poor uncertainty estimation (see Section \ref{subsection:bilipschitz}).
Both methods use a pre-training phase with a standard, non GP, output, and SV-DKL trains with a mini-batch size of 5,000.
GPDNN trains with different optimizers for pre-training and training with GP output, and uses different learning rates for variational versus model parameters.
In summary, these models are more cumbersome to use than alternative methods for uncertainty estimation, such as Deep Ensembles \citep{lakshminarayanan2017simple}.
\section{METHOD}
Here we explain how feature collapse affects DKL's uncertainty, give insight into why DKL's objective leads to feature collapse, and use that insight to mitigate feature collapse by introducing the DUE model.
\subsection{Feature Collapse}
\label{subsection:bilipschitz}
When the deep feature extractor inside the kernel is unconstrained, it can map in- and out-of-distribution data to the same location in feature space.
The GP assigns high confidence to inputs similar to the training data, so when in- and out-of-distribution data have the same feature representation, the GP will assign high confidence to out-of-distribution inputs.
This is called \textit{feature collapse} \citep{van2020uncertainty}, and we visualize the pathology in Figure \ref{fig:feature_collapse}.

The objective in DKL directly encourages feature collapse:  the ELBO (see Appendix Equation \ref{eq:elbo}) consists of an expected log-likelihood term and a KL term (``data fit'' and ``complexity penalty'' respectively \citep{rasmussen2006gaussian}).
At convergence, the ``data fit'' term tends to $-\frac{N}{2}$, leaving only the ``complexity penalty'': $\log |K(f_\theta(X),f_\theta(X)) + \sigma^2 I |$.
Since $\sigma$ depends on the scale of observation noise and cannot usually be set to zero, to minimize the penalty term we must minimize the log determinant of the covariance matrix.
It is this minimization of the term $|K(f(X),f(X))|$ which leads to feature collapse: the determinant tends to zero for feature representations of X that are collinear (intuitively, features which are mapped close together in feature space up to some constant scale).
When optimising the feature extractor $f_\theta$ as part of the objective, this leads to feature extractors that collapse features.
We formalize this statement and provide a proof in Appendix \ref{appendix:complexity_penalty}.
This behavior is discussed in more detail in \citet{ober2021promises}, who also point out that it can lead to worse overfitting than standard maximum likelihood training.

Feature collapse can be reduced by enforcing two constraints on the model: sensitivity and smoothness \citep{van2020uncertainty}.
Sensitivity implies that when the input changes the feature representation also changes.
Thus the model cannot simply collapse feature representations arbitrarily.
Smoothness implies small changes in the input cannot cause massive shifts in the output.

\begin{algorithm}[t]
  \caption{Algorithm for training DUE}
  \label{alg:DUE}
  \begin{algorithmic}[1]
    \STATE \textbf{Definitions:}

    - Residual NN $f_\theta: x \rightarrow \mathbb{R}^J$ with feature space dimensionality J and parameters $\theta$.

    - Approximate GP with parameters $\phi = \{l, s, \omega\}$, where $l$ length scale and $s$ output scale of $\bar{k}$, $\omega$ GP variational parameters (including $m$ inducing point locations $Z$)

    - Learning rate $\eta$, loss function $\mathcal{L}$
    \STATE Using a random subset of $p$ points of our training data, $X^\text{init} \subset X$, compute:

    \textbf{Initial inducing points:} K-means on $f_\theta(X^\text{init})$ with $K=m$. Use found centroids as initial inducing point locations Z in GP.

    \textbf{Initial length scale:} \\ $l = \frac{1}{\binom{p}{2}} \sum_{i=0}^p\sum_{j=i+1}^{p} |f(X_i^{\text{init}}) - f(X_j^{\text{init}})|_2$.

    \FOR{minibatch $\vx_b, \vy_b \subset X, Y$}
    \STATE $\theta' \leftarrow \text{spectral\_normalization}(\theta)$
    \STATE $p(\vy_b'|\vx_b) \leftarrow \text{evaluate\_GP}_\phi(f_{\theta'}(\vx_b))$
    \STATE $\mathcal{L} \leftarrow \text{ELBO}_\phi(p(\vy_b'|\vx_b), \vy_b)$
    \STATE $(\phi, \theta) \leftarrow (\phi, \theta) + \eta * \nabla_{\phi, \theta} \mathcal{L}$
    \ENDFOR
  \end{algorithmic}
\end{algorithm}
These constraints can be achieved by enforcing the feature extractor to be bi-Lipschitz, which has also been shown to help in a large number of other contexts \citep{rosca2020case}.
A bi-Lipschitz feature extractor can be obtained in different ways and each comes with different trade-offs:
\begin{itemize}
  \item \textbf{Two-sided gradient penalty:} this penalty was introduced in the context of GANs \citep{gulrajani2017improved}. It regularizes using a two-sided gradient penalty, penalizing the squared distance of the gradient from a fixed value at every input point. It is easy to implement, but only a soft constraint. In practice the training stability and regularization effectiveness is sensitive to the relative weight of the penalty in the loss. This method is used in DUQ \citep{van2020uncertainty}.
  \item \textbf{Direct spectral normalization and residual connections:} spectral normalization \citep{miyato2018spectral, gouk2018regularisation} on the weights leads to smoothness and it is possible to combine this with an architecture that contains residual connections for the sensitivity constraint. This method is faster than the gradient penalty, and in practice offers a more effective way of mitigating feature collapse. This method is used in SNGP \citep{liu2020simple}.
  \item \textbf{Reversible model:} A reversible model is constructed by using reversible layers and avoiding any down scaling operations \citep{jacobsen2018revnet, behrmann2019invertible}. This approach can guarantee that the overall function is bi-Lipschitz, but it consumes considerably more memory and can be difficult to train.
\end{itemize}
In this work, we use direct spectral normalization and residual connections, as we find it to be more stable than a direct gradient penalty and significantly more computationally efficient than reversible models.
In Figure \ref{fig:feature_collapse}, we show that a constrained model does not collapse points on top of each other in feature space, enabling the GP to correctly quantify uncertainty.
Using a stationary kernel, the model reverts back to the prior away from the training data just like a GP defined over the input space.
\begin{figure*}[t]
  \centering
  \begin{subfigure}[c]{.3\linewidth}
    \includegraphics[width=\linewidth]{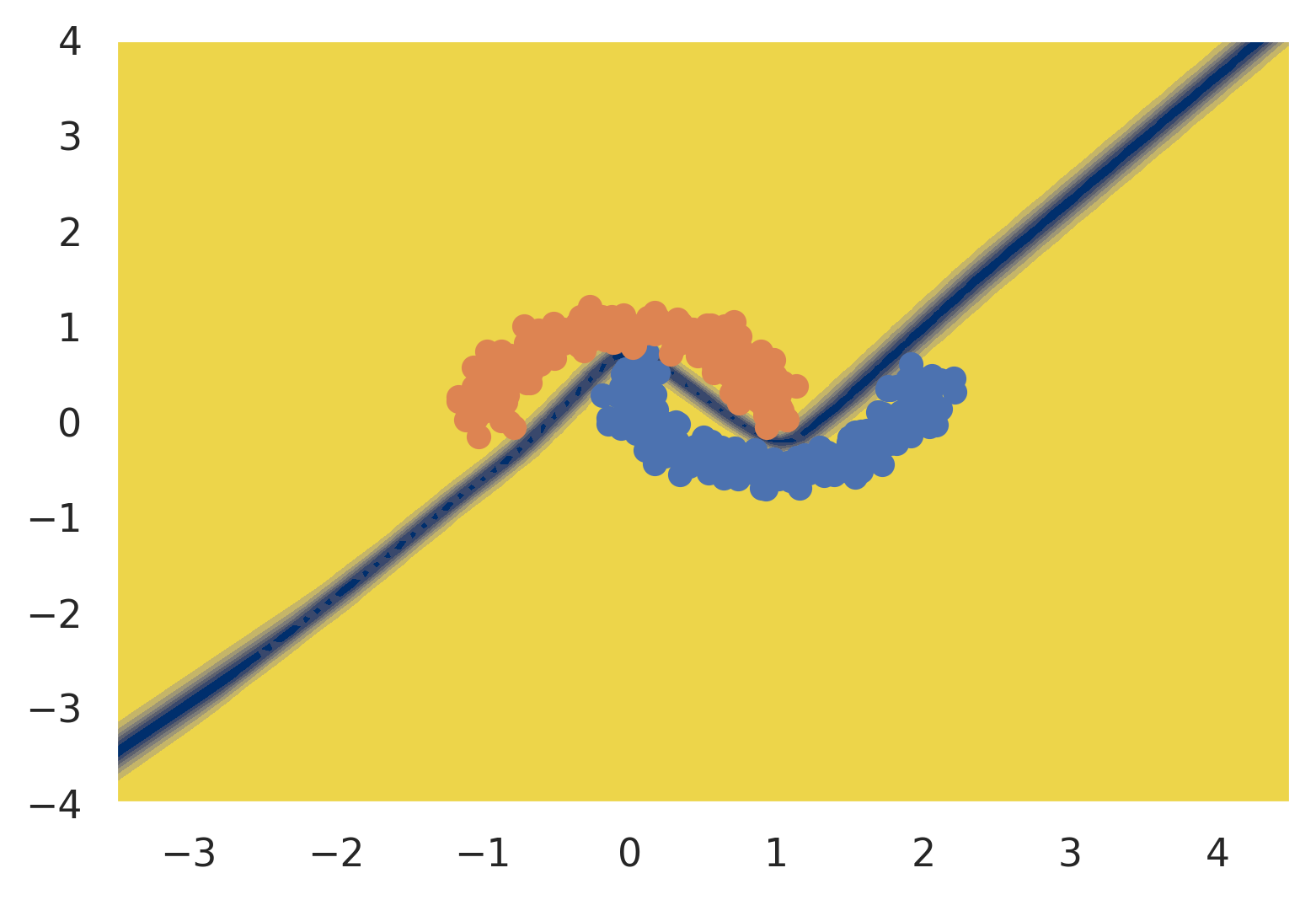}
    \caption{ResNet + Softmax}
    \label{fig:twomoons_softmax}
  \end{subfigure}
  \begin{subfigure}[c]{.3\linewidth}
    \includegraphics[width=\linewidth]{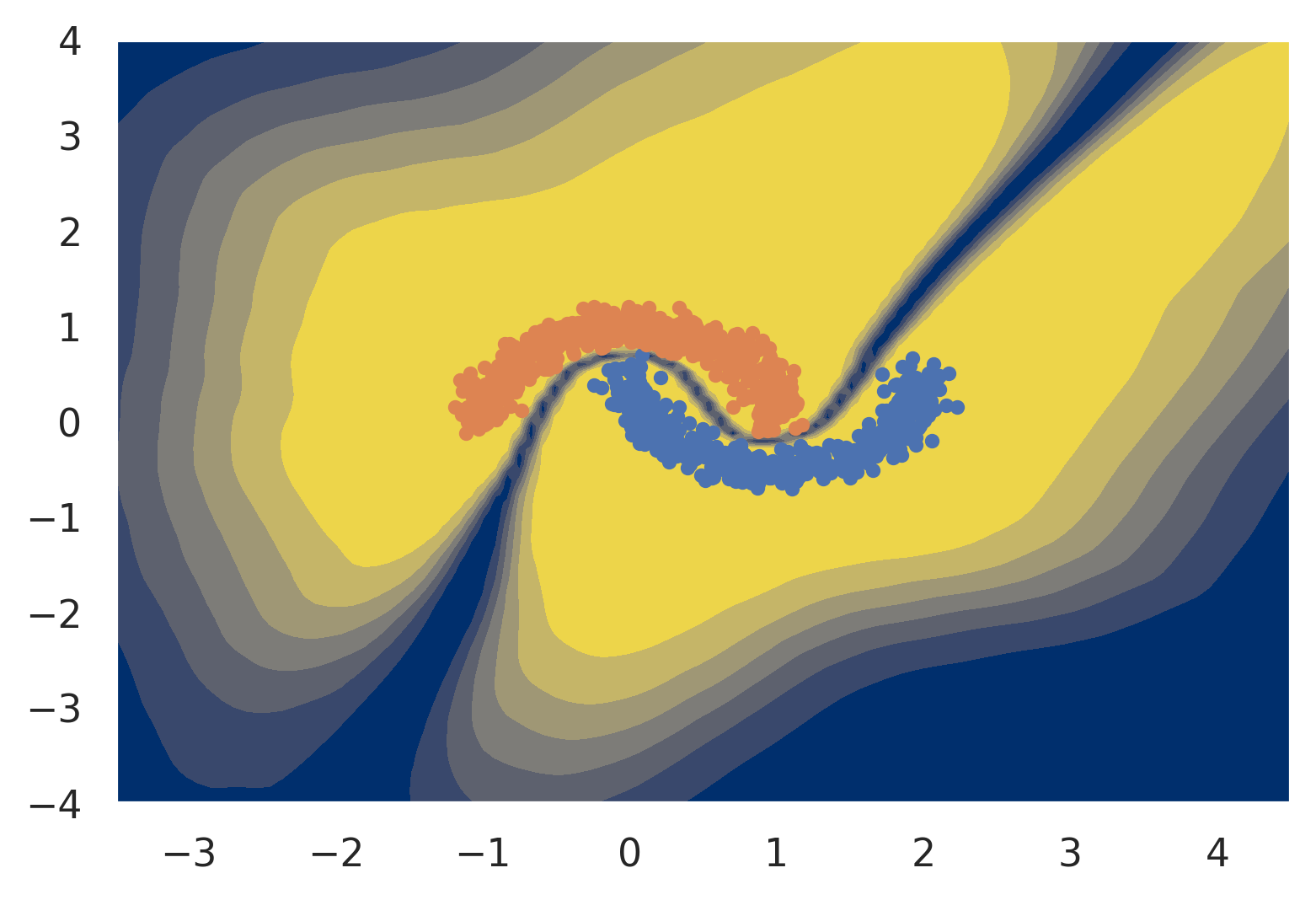}
    \caption{GPDNN \citep{bradshaw2017adversarial}}
    \label{fig:twomoons_ffn}
  \end{subfigure}
  \begin{subfigure}[c]{.3\linewidth}
    \includegraphics[width=\linewidth]{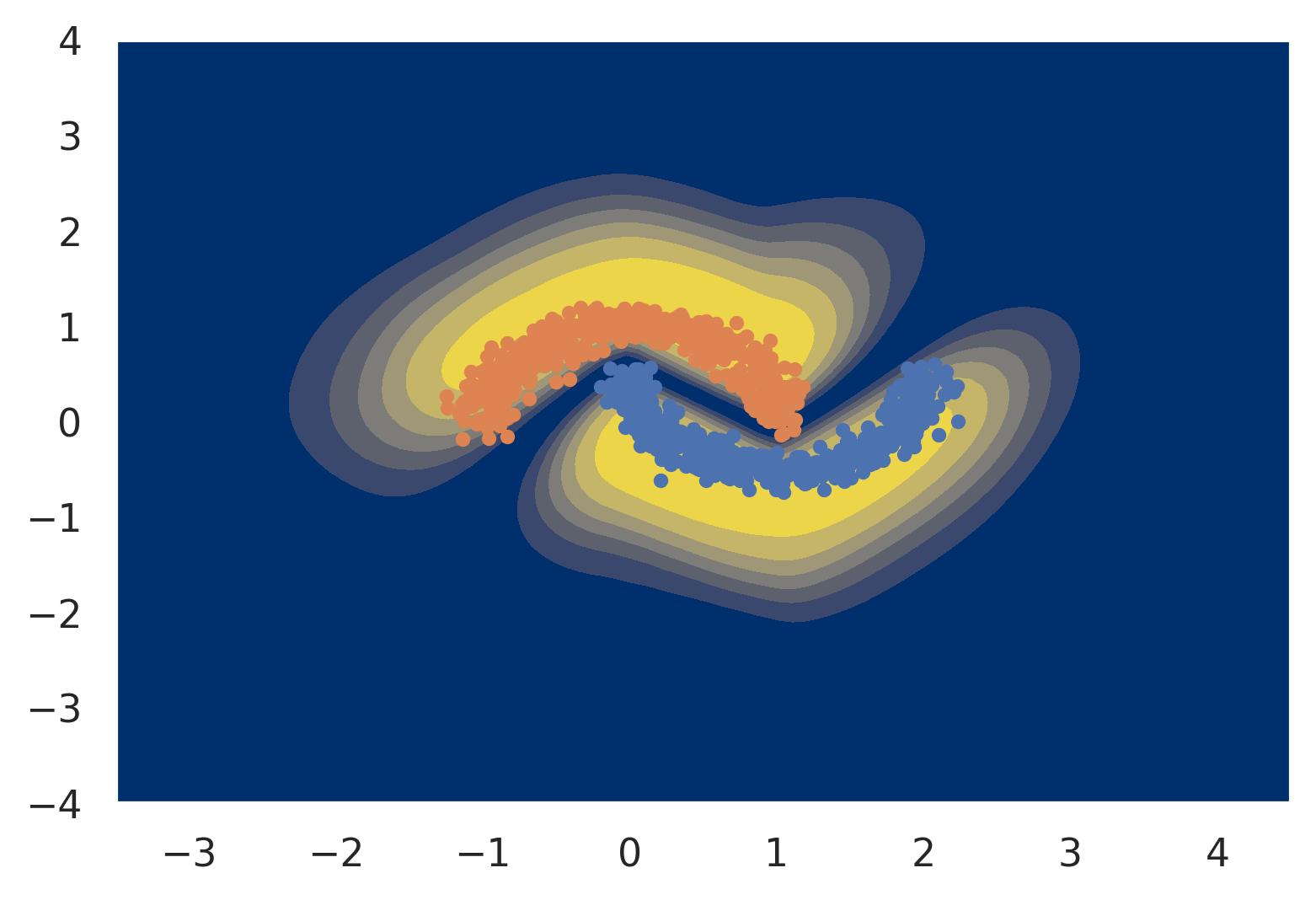}
    \caption{DUE (us)}
    \label{fig:twomoons_DUE}
  \end{subfigure}
  \caption{
    We show uncertainty results on the two moons dataset.
    Yellow indicates high confidence, while blue indicates uncertainty.
    In Figure \ref{fig:twomoons_softmax}, a simple feed-forward ResNet with a softmax output is certain everywhere except on the decision boundary.
    In Figure \ref{fig:twomoons_ffn}, we see that GPDNN, which uses a simple Feed-Forward Network as feature extractor, is certain even far away from the training data.
    In Figure \ref{fig:twomoons_DUE}, we show DUE, which has the appropriate restrictions on the feature extractor (residual connections and spectral normalization) and obtains close to ideal uncertainty on this dataset.
  }
  \label{fig:twomoons}
\end{figure*}
\subsection{Model}
\label{section:method}
\textbf{DUE} builds upon the foundation of GPDNN \citep{bradshaw2017adversarial}.
To enforce the bi-Lipschitz constraint and enable high quality uncertainty, we restrict the deep feature extractor to have residual connections in combination with spectral normalization as described in the previous section.
We further make several simplifications to the training process to make GPDNN more practical to use.
Instead of an additional downsampling layer to 25 dimensions which increases the risk of feature collapse, the GP output is placed directly on top of the last convolutional layer of a large model, which is 640 dimensional in the case of the WRN.
In practice, DUE trains well with just 10 inducing points on CIFAR-10 which leads to a runtime that is only 3\% slower than a softmax model.
No pre-training is necessary and training is stable with a single set of hyper-parameters for both the variational and model parameters.
The inference procedure is described in detail in Appendix \ref{appendix:model_details}.
The training procedure is described in Algorithm \ref{alg:DUE}.

Online power iteration for spectral normalization can be implemented exactly for fully connected layers and 1x1 convolutions.
For convolutions larger than 1x1, we use an approximate method that lower bounds the exact method, as proposed in \citet{gouk2018regularisation} and implemented by \citet{behrmann2019invertible}.

\textbf{Batch Normalization} In SNGP, spectral normalization is applied only on the convolution operations.
However batch normalization, a crucial component of training deep models, has a non-trivial Lipschitz constant.
Batch normalization transforms the input following:
\begin{align}
  x_{out} = \mathrm{diag}\left(\frac{\gamma}{\sqrt{ \mathrm{Var}(x)}}\right)(x - \mathbb{E}[x]) + \beta,
\end{align}
with $\gamma$ and $\beta$ the learnable scale and shift parameters.
It has a Lipschitz constant of $\max_i | \frac{\gamma_i}{\sqrt{\mathrm{Var}(x)}_i}|$ \citep{gouk2018regularisation}.
Using the above equation, we can extend spectral normalization to batch normalization by dividing the weight $\gamma$ of the batch normalization by the (scaled) Lipschitz constant.
In practice, we find that batch normalization layers in trained ResNets have a relatively high Lipschitz constant, up to around 12, and $95\%$ of the channel-wise Lipschitz constants are greater than one (see also Figure \ref{fig:lipschitz_histogram} in the Appendix).
Adding spectral normalization to batch normalization layers ensures the entire network's upper Lipschitz constant is bounded.
\begin{table*}[t]
  \centering
  \caption{
    AUROC on CIFAR10 vs SVHN for two DKL methods trained with feature extractors with and without a bi-Lipschitz constraint.
    Non bi-Lipschitz uses VGG-19 \citep{simonyan2014very}; bi-Lipschitz uses a WRN with spectral normalization.
    Both models obtain high accuracy, matching standard softmax models.
    SV-DKL without bi-Lipschitz obtains poor uncertainty: distinguishing in- and out-of-distribution data is no better than chance.
    GPDNN without bi-Lipschitz obtains better but still poor uncertainty.\linebreak
    The cell highlighted in gray is DUE.
  }
  \label{table:feature_extractor}
  \definecolor{light-gray}{gray}{0.93}
  \begin{tabular}{lcccc}
    \toprule
                                          & bi-Lipschitz                                  &                                       & non bi-Lipschitz &               \\
                                          & AUROC                                         & Accuracy                              & AUROC            & Accuracy      \\ \hline
    SV-DKL \citep{wilson2016stochastic}   & \textbf{0.959$\pm$.001}                       & 95.7$\pm$0.06                         & 0.498$\pm$.001   & 93.6$\pm$0.05 \\
    GPDNN \citep{bradshaw2017adversarial} & \cellcolor{light-gray}\textbf{0.958$\pm$.005} & \cellcolor{light-gray}{95.6$\pm$0.04} & 0.876$\pm$.004   & 93.7$\pm$0.10 \\
    \bottomrule
  \end{tabular}
\end{table*}
\section{RELATED WORK}
\label{sec:related_work}
The single forward pass uncertainty methods most similar to DUE are DUQ \citep{van2020uncertainty} and SNGP \citep{liu2020simple}, as discussed in the introduction and in the context of enforcing the bi-Lipschitz constraint.
Compared to DUQ, DUE can readily be applied to regression problems, and the gradient penalty in DUQ is difficult to optimize and slow to compute.
SNGP \citep{liu2020simple} consists of a deep feature extractor and an output GP, which is approximated using the parametric Random Fourier Features (RFF) approximation \citep{rahimi2008random} in combination with a Laplace approximation for the non-conjugate Softmax likelihood \citep{rasmussen2006gaussian}.
We discuss the trade-offs between the RFF and variational inducing point approximations in the next subsection.
\subsection{Inducing Points versus RFF Approximation}
\label{subsection:approximations}
Inducing point GP approximations maintain the non-parametric properties of the full GP, while the RFF approximation sacrifices this.
With the RFF approximation, for any \emph{finite} number of features, the kernel is approximated as a linear model on a finite number of features.
Because of this, the RFF GP's uncertainty will erroneously concentrate to zero as the number of training points increases even in areas where there is no training data.
Various extensions of RFF (and its close relative, the sparse spectrum GP approximation) have attempted to fix this problem, see for example \citet{gal2015latent}.

Inducing point GP approximations, on the other hand, use as the approximation a standard GP which is defined over a set of \emph{inducing inputs} instead of over the entire training set (which is computationally prohibitive).
These do not change the kernel definition (unlike the RFF approximation), and the GP is still non-parametric.
The inducing point approximation then tries to move the inducing input locations to minimize the KL between this inducing point GP and the full GP \citep{titsias2009variational}, giving a tractable objective which also preserves the full GP's non-parametric properties.
The approximation will become exact when the number of inducing points matches the number of training points (at which point the ELBO can match the GP marginal likelihood by placing the inducing inputs on the training input locations).
This is in contrast to RFF which will only become exact at the infinite limit of the number of random features.
Inducing point GPs also have a tight bound on the marginal log-likelihood, which means that we can do model selection by optimising the ELBO with respect to various hyper-parameters \citep{burt2019rates}.

In Figure \ref{fig:rff}, we show the effects of the approximations in DUE and SNGP in practice by training on a small and large dataset sampled from the same distribution.
SNGP's uncertainty interval is dependent on the number of training points: it is wide in areas where there is no training data at 1K datapoints, but very narrow in the same regions at 1M datapoints.
While DUE's inducing point GP has similar uncertainty outside the support of the training data when trained on either dataset size.
SNGP uncertainty was estimated using the exact method with a ridge penalty of 1.
All other hyper-parameters were shared between the two methods, and are listed in Appendix \ref{appendix:experimental_details}.
\section{EXPERIMENTS}
\subsection{Feature Collapse in DKL}
\label{experiments:twomoons}
We analyze the problem of feature collapse in DKL in Figure \ref{fig:twomoons} on the Two Moons dataset \citep{pedregosa2011scikit}.
We show results for three different models: a standard softmax model, DUE, and a variation where the spectral normalized ResNet is replaced by a fully connected model (similar to \citet{bradshaw2017adversarial}), further details are provided in Appendix \ref{appendix:experimental_details}.
The uncertainty is computed using the predictive entropy of the class prediction; we model the problem as a two class classification problem.

\begin{table*}[t]
  \centering
  \caption{
    Results on the CIFAR-10 dataset, and distinguishing between CIFAR-10 and SVHN by uncertainty.
    All results use a WRN as feature extrator and are the average and standard error of 5 runs.
    The runtime is relative to the ``Standard WRN'' row.
    In bold are top results (within standard error), and the horizontal line separates ensembles from single forward pass methods.
  }
  \label{table:cifar10}
  \begin{tabular}{lccc}
    \toprule
    Method                                                 & Runtime $\downarrow$ & Accuracy (\%)    $\uparrow$ & AUROC $\uparrow$         \\ \hline
    Ensemble of 5 \citep{lakshminarayanan2017simple}       & 5x                   & \textbf{96.6$\pm$0.03}      & \textbf{0.967$\pm$0.005} \\ \hline
    Standard WRN \citep{zagoruyko2016wide}                 & \textbf{1x}          & \textbf{96.2$\pm$0.01}      & 0.932$\pm$0.008          \\
    DUQ \citep{van2020uncertainty}                         & 3.5x                 & 94.9$\pm$0.04               & 0.940$\pm$0.003          \\
    SNGP \citep{liu2020simple}                             & \textbf{1x}          & 96.0$\pm$0.04               & 0.940$\pm$0.006          \\
    SV-DKL (with constraints) \citep{wilson2016stochastic} & 2x                   & 95.7$\pm$0.06               & \textbf{0.959$\pm$0.001} \\
    \textbf{DUE}                                           & \textbf{1x}          & 95.6$\pm$0.04               & \textbf{0.958$\pm$0.005} \\
    \bottomrule
  \end{tabular}
\end{table*}
The standard softmax output model is certain everywhere except on the decision boundary.
DUE (Figure \ref{fig:twomoons_DUE}) quantifies uncertainty as expected for the two moons dataset: certain on the training data, uncertain away from it and in-between the two moons.
Figure \ref{fig:twomoons_ffn} highlights the importance of our contribution, because the uncertainty estimation of a standard Feed-Forward Network in combination with DKL suffers from feature collapse and is certain away from the training data.
\subsection{Feature Collapse in CIFAR-10 versus SVHN}
\label{experiments:cifar_svhn}
We now consider training end-to end with a large feature extractor, the Wide Residual Network (WRN), on CIFAR-10 \citep{krizhevsky2009learning}.
We follow \citet{zagoruyko2016wide} and use a 28 layer model with BasicBlocks and dropout.
Importantly, we can use their hyper-parameters (such as dropout rate, learning rate, and optimizer) when training DUE, and no further tuning is necessary.
We remove the final linear layer of the original WRN and the $640$-dim feature vector is directly used in the GP.
We use 10 inducing points for DUE on CIFAR-10, which leads to a runtime of 1m47s for one epoch on an Nvidia GTX1080 Ti, while a standard WRN with a linear+softmax output takes 1m43s.
An ablation were we look at the learned location of the 10 inducing points is available in the Appendix Figure \ref{fig:inducing_points}, and using more inducing points is evaluated in Table \ref{table:inducing_ablation}.

\textbf{Uncertainty.} To compare the uncertainty quality between different models, we use the experiment of distinguishing between the test sets of CIFAR-10 and SVHN \citep{netzer2011reading} -- a notably difficult dataset pair \citep{nalisnick2018deep} -- using each model's uncertainty metric.
In Table \ref{table:feature_extractor}, we compare SV-DKL and GPDNN with two different feature extractors: with and without the DUE constraints.
In contrast to the original SV-DKL and GPDNN, these models were trained from scratch using the same hyper parameters (including mini-batch-size) as DUE.
Note that GPDNN with constraints and the improved training setup is conceptually the same as DUE.
We see that in both cases the accuracy is high, but the uncertainty is very poor when no constraints are placed.
This highlights the importance of our contribution: additional constraints on the feature extractor are crucial for good uncertainty performance in DKL.

Table \ref{table:cifar10} compares DUE to several baselines.
We train DUE using 10 inducing points and spectral normalization constant 3, set in similar fashion to \citet{liu2020simple} by choosing the lowest value that does not significantly affect accuracy.
All methods compute uncertainty using the predictive entropy, except DUQ which uses the closest kernel distance \citep{van2020uncertainty}.
We see that DUE and SV-DKL, with DUE feature extractor, outperform all competing single forward pass uncertainty methods.
DUE is competitive with Deep Ensembles uncertainty, but only requires a single forward pass and is therefore five times faster to train.
SV-DKL with the DUE feature extractor performs similar to DUE, but takes 2x longer to train, since it needs to invert a larger covariance matrix.
The DUE and SV-DKL result shows that the theoretically preferable inducing point variational approximation also leads to practical improvements, as both outperform SNGP in terms of uncertainty estimation.

\begin{figure}
  \centering
  \begin{subfigure}{0.5\linewidth}
    \centering
    \includegraphics[width=\linewidth]{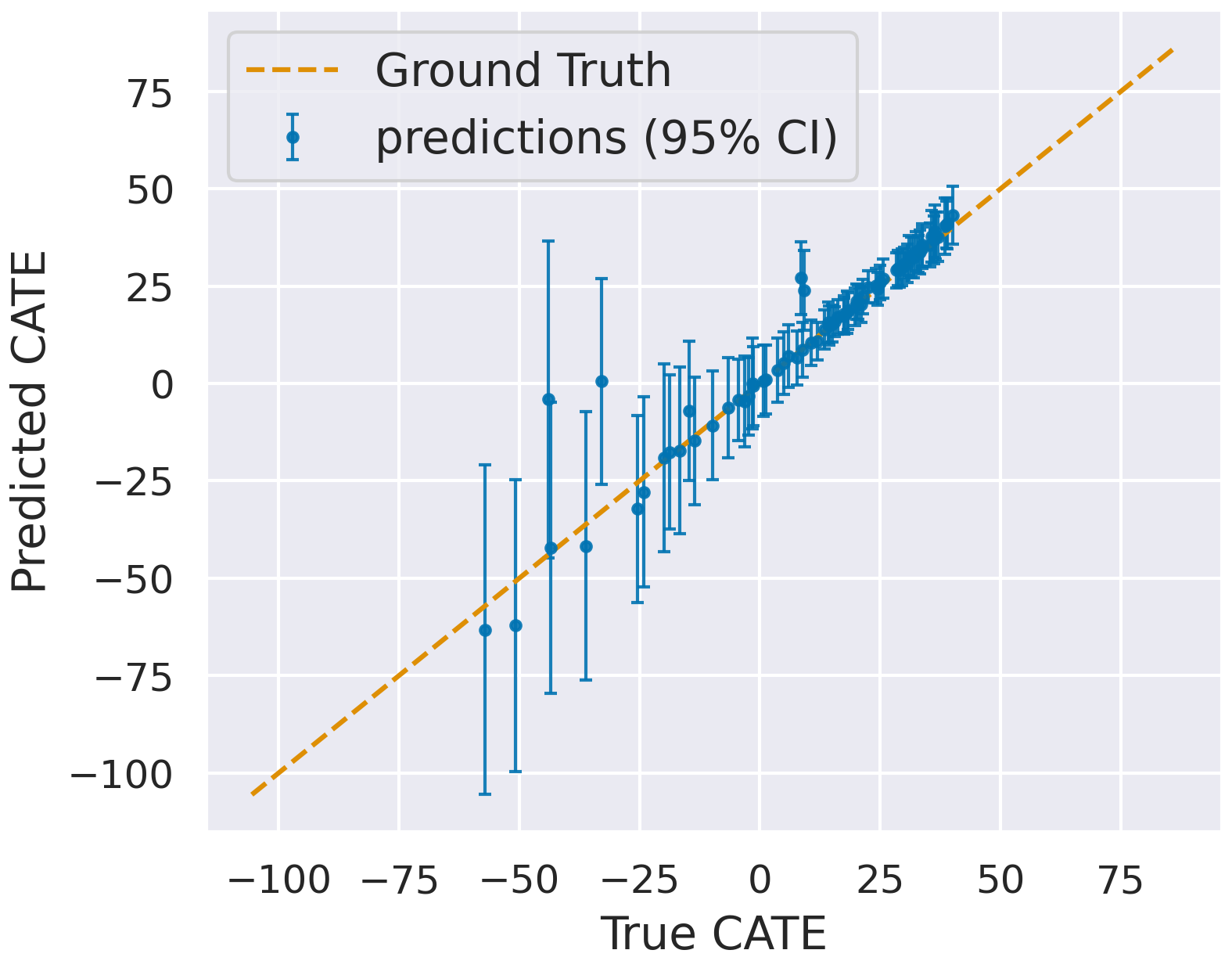}
    \caption{BTARNet}
    \label{fig:causal_btarnet}
  \end{subfigure}%
  \begin{subfigure}{0.5\linewidth}
    \centering
    \includegraphics[width=\linewidth]{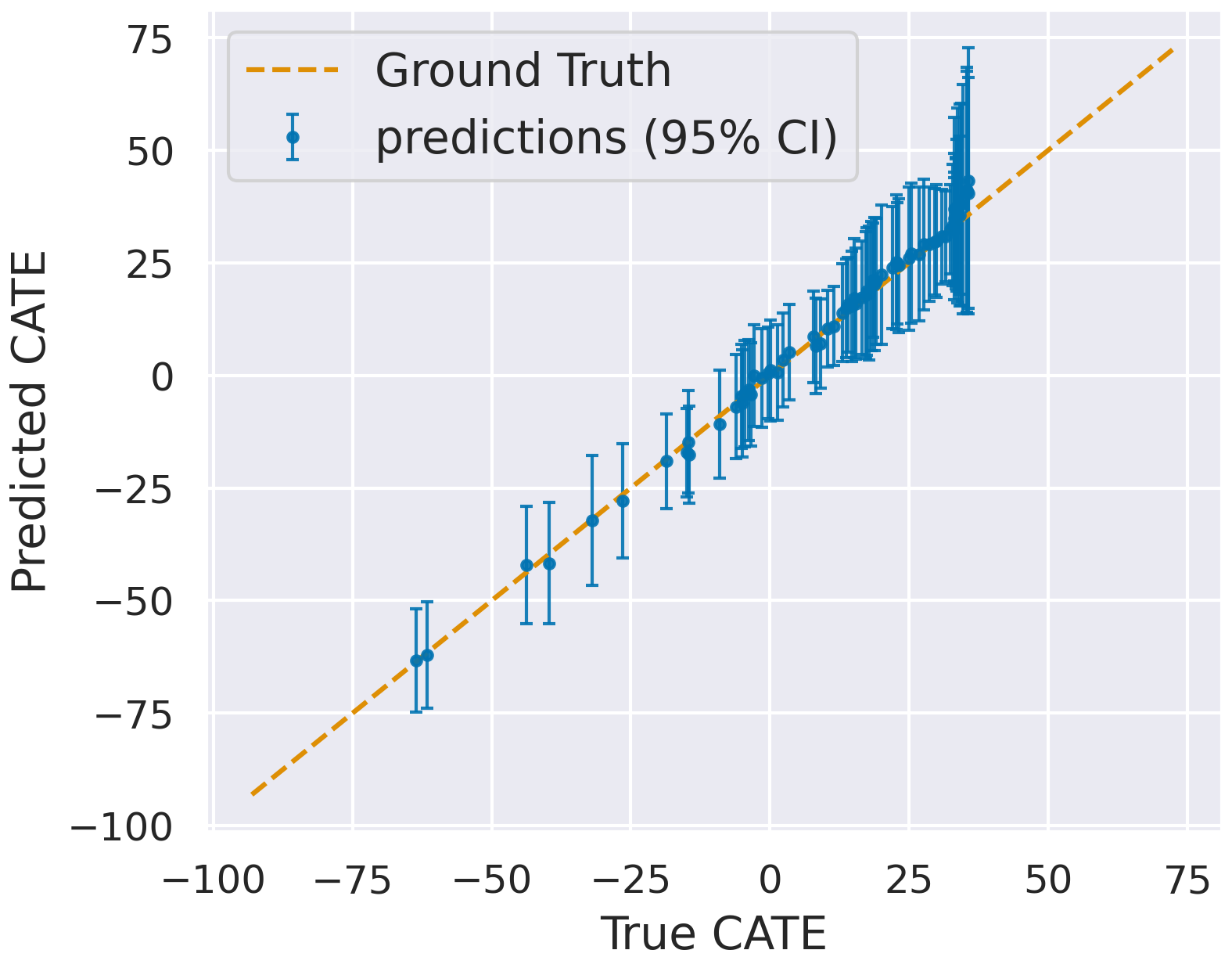}
    \caption{DUE}
    \label{fig:causal_DUE}
  \end{subfigure}
  \caption{
    Predicted CATE versus true CATE with 95\% confidence intervals for a randomly chosen cross-validation run.
    DUE is confident (without converging to no uncertainty) and correct, while BTARNet \citep{shalit2017estimating} is wrong in $2$ instances and the true CATE is not within the confidence interval.
  }
  \label{fig:causal_uncertainty}
\end{figure}

\begin{table*}[t]
  \setlength\tabcolsep{3pt} 
  \centering
  \caption{
    Comparing the performance in terms of RMSE of several treatment effect and uncertainty estimation models under ``random'' and ``uncertainty'' based deferral to an expert (50\% and 10\% deferred for IHDP Cov. and IHDP respectively).
    The first three rows were obtained from \citet{jesson2020identifying}, DKLITE was evaluated using the author's open source implementation and the ensemble of TARNet was reimplemented.
    DUE outperforms all alternative methods, while being 5x faster to train and evaluate than the ensemble.
  }
  \begin{tabular}{lcc|cc}\toprule
                                           & \multicolumn{2}{c|}{IHDP Cov. (50\% def.)} & \multicolumn{2}{c}{IHDP (10\% def.)}                                                  \\
    Method                                 & \emph{random}                              & \emph{uncertainty}                   & \emph{random}         & \emph{uncertainty}     \\ \midrule

    BART       \citep{chipman2010bart}     & 2.6$\pm$.2                                 & 1.8$\pm$.2                           & 1.90$\pm$.20          & 1.60$\pm$.10\          \\
    BTARNet   \citep{shalit2017estimating} & 2.2$\pm$.3                                 & 1.2$\pm$.1                           & 1.10$\pm$.03          & 0.76$\pm$.03\          \\
    BCEVAE \citep{louizos2017causal}       & 2.5$\pm$.2                                 & 1.7$\pm$.1                           & 1.80$\pm$.06          & 1.47$\pm$.08\          \\
    DKLITE \citep{zhang2020learning}       & 2.6$\pm$.7                                 & 1.8$\pm$.5                           & 1.74$\pm$.53          & 1.34$\pm$.41\          \\
    Ensemble of 5 TARNet                   & 1.74$\pm$.1                                & 1.19$\pm$.03                         & 1.14$\pm$.04          & 0.76$\pm$.01           \\
    \textbf{DUE}                           & \textbf{1.63$\pm$.06}                      & \textbf{1.05$\pm$.05}\               & \textbf{0.91$\pm$.04} & \textbf{0.48$\pm$.02}\ \\ \bottomrule
  \end{tabular}
  \label{table:causal}
\end{table*}
\subsection{Uncertainty in Regression for Personalized Healthcare}
\label{section:causal}
To demonstrate the performance of DUE on a regression task, we focus on a new benchmark on personalized healthcare \citep{jesson2020identifying}.
The task is to predict responses of individuals to a particular treatment, which is only possible if there is sufficient data available to assess how they might respond.
This is also called the \emph{overlap assumption}, which states that to predict the effect of treatment for a particular input $x$ (an individual described by features), we need to have seen similar data points that have received treatment as well as points that have not received treatment.
We can use uncertainty to assess when this assumption is violated and the patient should be referred to an expert instead \citep{jesson2020identifying}.
It is important that the uncertainty estimates are accurate, otherwise an individual could receive treatment even if their response to treatment is not known, which can result in undue stress, financial burdens, or worse.

This benchmark assesses both predictive performance and uncertainty estimation.
Other benchmarks which only evaluate test log-likelihood, such as UCI \citep{dua2017uci}, only capture uncertainty on the test set (in-distribution), which is insufficient to evaluate uncertainty on out-of-distribution data.
In the case of feature collapse, the test log-likelihood is therefore not a useful comparison tool.

We train DUE on the input features of the training set, appending a $0$ or $1$ to represent no treatment and treatment.
The model is then used to predict the Conditional Average Treatment Effect (CATE) \citep{abrevaya2015estimating}, computed as the difference between the expected effect of treatment and no treatment.
The CATE estimate and its uncertainty are the expectation and variance over the difference in joint predictions with the two inputs:
\begin{align*}
  [y_0, y_1]              \sim                         & \text{ DUE}([\vx, t=0], [\vx, t=1])                           \\
  \text{CATE}(x)               = \mathbb{E}[y_1 - y_0] & \qquad       \text{CATE}_\text{u}(x) = \text{Var}[y_1 - y_0].
\end{align*}
The CATE can be computed exactly, using the mean of the GP posterior, while we use Monte Carlo sampling of the joint posterior for the uncertainty (note that this requires only a single forward pass through the model, and sampling from the GP is fast).
We use the uncertainty to decide which predictions will be deferred to an expert.
This process allows us to make a \emph{causal} statement on the effect of treatment, under the assumptions in \citet{jesson2020identifying}.

We use the IHDP \citep{hill2011bayesian} and IHDP Covariate shift (referred to as IHDP Cov.) datasets.
IHDP is a regression datasets with $\sim$750 data points, and IHDP Cov. is a variant with additional covariate shift to increase the difficulty of the task.
These are real world datasets derived from the Infant Health and Development Program.
The details of the covariate shift are discussed in \citet{jesson2020identifying}.
In the experiments, treatment-effect recommendations are deferred to an expert if the CATE estimate has high uncertainty.
We include a baseline that defers at random.
We run IHDP and its covariate shift variant for 1,000 cross-validation trials.

In Figure \ref{fig:causal_uncertainty}, we compare DUE with Bayesian TARNet, which is the standard TARNet \citep{shalit2017estimating} extended with MC dropout \citep{gal2016dropout} for uncertainty quantification.
The TARNet is a commonly used deep learning baseline in the causality field.
The results show that DUE is more accurate than BTARNet for most CATE values, and that the BTARNet makes predictions for which the ground truth is not within the confidence interval.
Table \ref{table:causal} summarizes our results compared to several baselines (detailed in Appendix \ref{appendix:causal_experiments}) and shows that DUE has improved performance and uncertainty estimates better suited to rejection policies than other uncertainty-aware methods.
\section{Conclusion and Limitations}
\label{sec:limitations}
We demonstrated that DUE outperforms alternative single forward pass methods on CIFAR-10, and obtains SotA performance in a personalized medicine regression benchmark.
These results show that DUE overcomes the previous problems with uncertainty in DKL, and makes DKL a viable method for uncertainty estimation and a practical tool for improving reliable AI.
The main limitation of DUE is that despite the empirical improvements, the uncertainty estimation is not guaranteed to be correct.
In particular, the downsampling operations are not bi-Lipschitz and future work is necessary to assess the practical implications of this.
A potential negative societal impact of this work is that if DUE is used in sensitive applications and incorrectly estimates its uncertainty, then decisions might be made in an automated fashion that should have been deferred to an expert.
We believe this risk is smaller than with other models, but should nevertheless be evaluated within the applied setting.

\clearpage

\subsubsection*{Acknowledgements}

The authors would like to thank the members of OATML, OxCSML and anonymous reviewers for their feedback during the project.
In particular, we would like to thank Tim Rudner, Mark van der Wilk, Sebastian Farquhar, Bas Veeling, Andreas Kirsch, and Luisa Zintgraf for fruitful discussions and their suggestions.
JvA/LS are grateful for funding by the EPSRC (grant reference EP/N509711/1 and EP/L015897/1, respectively).
JvA is also grateful for funding by Google-DeepMind.

\bibliographystyle{plainnat}
\bibliography{due}

\clearpage
\onecolumn
\appendix

\begin{figure}[t]
  \begin{minipage}{.48\textwidth}
    \centering
    \includegraphics[width=\linewidth]{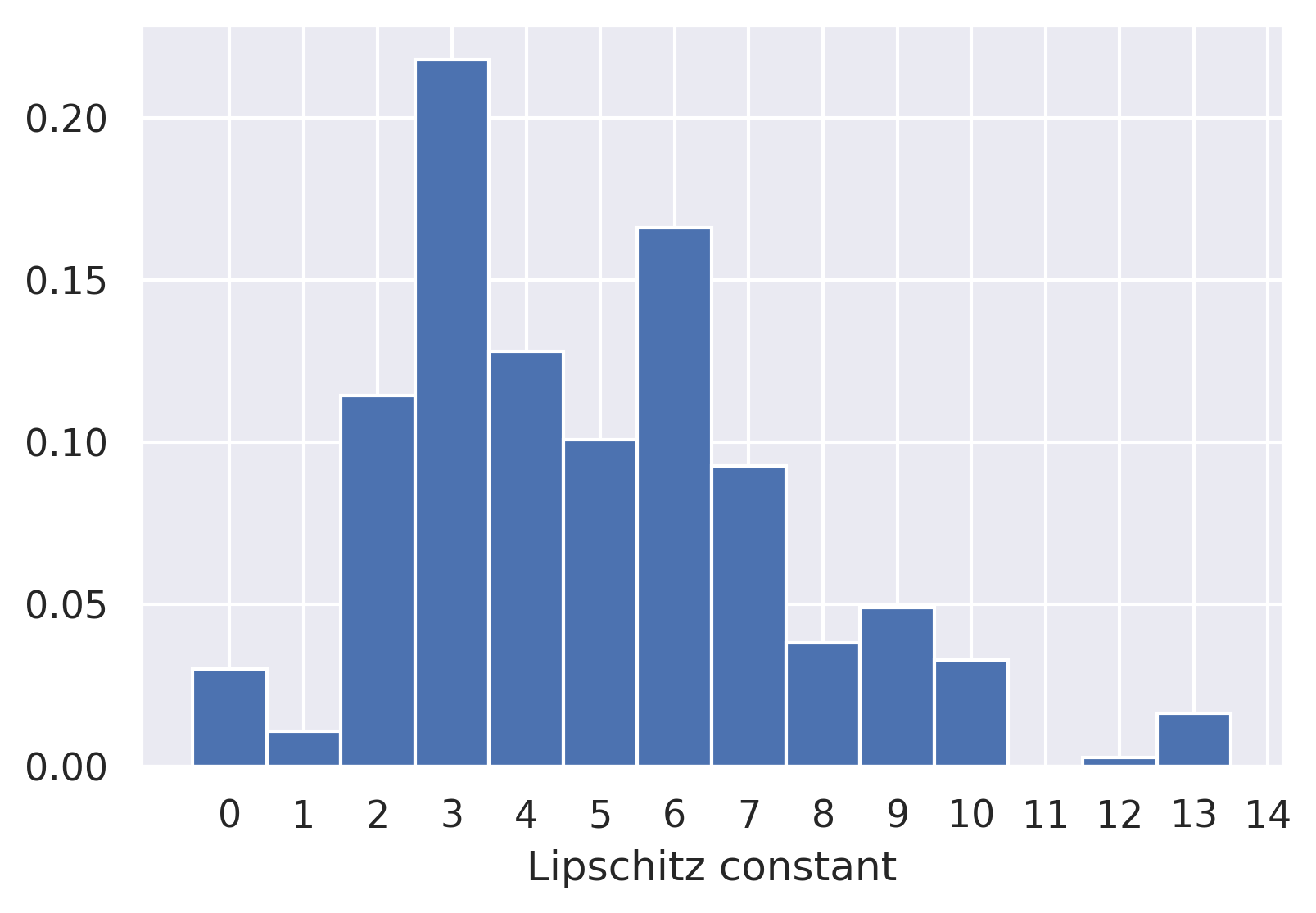}
    \captionof{figure}{
      A density of the Lipschitz values in batch normalization layers, averaged across 15 WRN models that were trained with Softmax output and without spectral normalization (exactly following \citet{zagoruyko2016wide}).
      We see that many of the constants are significantly above 1, highlighting that batch normalization has significant impact on the Lipschitz constant of the network.}
    \label{fig:lipschitz_histogram}
  \end{minipage}\hfill
  \begin{minipage}{.48\textwidth}
    \centering
    \includegraphics[width=\linewidth]{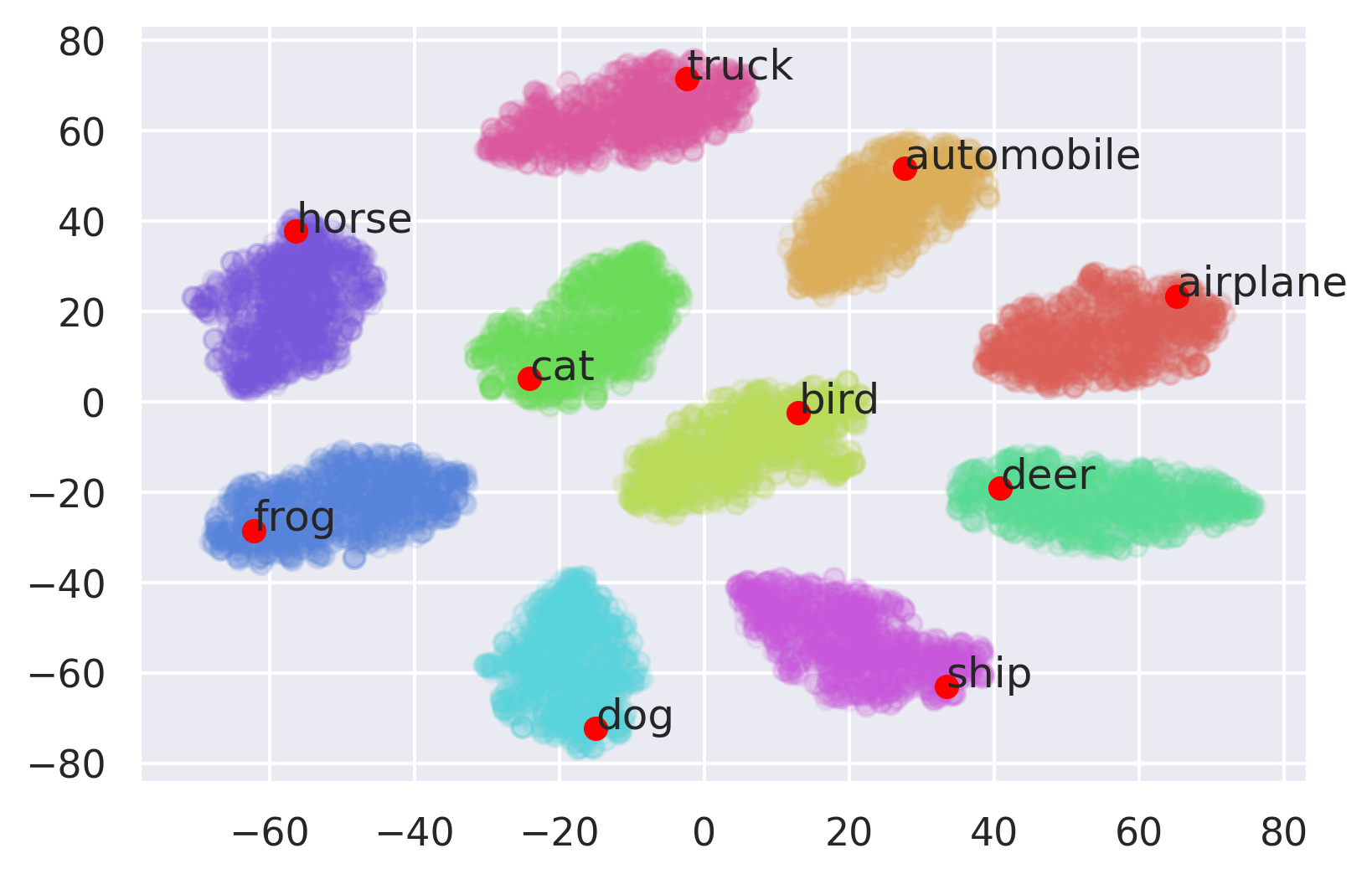}
    \captionof{figure}{
      Visualization of the learned feature representations of the training set of CIFAR-10.
      We use T-SNE \citep{van2008visualizing} to reduce the dimensionality down to 2D and we color the points by their class label.
      We overlay (in red) the inducing points location, which we label by computing the closest class average T-SNE representation.
      The features and inducing point locations are obtained from a DUE model trained with 10 inducing points and spectral normalization.
    }
    \label{fig:inducing_points}
  \end{minipage}
\end{figure}

\section{MODEL DETAILS}
\label{appendix:model_details}
DUE is an instance of DKL \citep{wilson2016deep} and uses the sparse GP of \citet{titsias2009variational} and the variational approximation of \citet{hensman2015scalable}.
In this section we give a complete description of the model.

\subsection{Full Model Definition}
Let $X \in \R^{N \times D}$ and $Y \in \R^{N \times T}$ be a dataset of $N$ points with input dimensionality $D$ and output dimensionality $T$.
In addition to regression, we also consider classification tasks.
For classification tasks, $T$ is the number of classes, with a single instance $\bm{y} \in Y$ being a $T$ dimensional vector of class probabilities.
Let $F \in \R^{N \times T}$ be the value of a $T$ dimensional latent function at each input.
For regression tasks $F$ is the values of the underlying noiseless function the GP is modelling.
The model is formed of an independent GP for each output dimension.
The joint distribution over $Y$ and $F$, evaluated at inputs $X$, is
\begin{align}
  p(Y, F; X)      & = p(Y \given F) \prod_{t=1}^T p(F_{[:,t]}; X) \\
  p(F_{[:,t]}; X) & = \GP(\mu_t(X), k_{l_t, \theta}(X, X)).
\end{align}
Here $F_{[:,t]}$ refers to the $t$th column vector of the matrix $F$.
$\mu_t(\cdot)$ is a mean function for output dimension $t$.
For both regression and classification we use a constant mean function $\mu_t(X) = \mu_t$, where $\mu_t$ is a hyperparameter.
$k_{l, \theta}(\cdot, \cdot)$ is a deep kernel function with feature extractor parameters $\theta$, as we discuss in Section \ref{section:method}, and base kernel $\bar{k}_{l_t}(\cdot, \cdot)$ with hyperparameters $l_t$ specific to each output dimension (but shared for each input dimension).
$p(Y \given F)$ is the likelihood function.
For regression tasks this is defined $p(Y \given F) = \prod_{i=1}^N \Normal(Y_{[i,:]} \given F_{[i,:]}, \sigma^2 I_T)$, where $\sigma^2$ is a variance hyparameter and $I$ is the identity matrix.
For classification tasks,
\begin{align}
  p(Y \given F)                 & = \prod_{i=1}^N p(Y_{[i,:]} \given F_{[i,:]})                \\
  p(Y_{[i,:]} \given F_{[i,:]}) & = \operatorname{softmax}(F_{[i,:]})_{(\argmax_c Y_{[i,c]})}.
\end{align}
Note that while $\mu_t$, $\sigma^2$, and $l_t$ are described as hyperparameters, we do not specify them manually but instead learn them alongside the other model parameters, as below.

\begin{table}[t]
  \begin{minipage}{.45\linewidth}
    \centering
    \caption{
      Test accuracy and Negative Log-Likelihood (NLL) on CIFAR-10 of DUE with WRN feature extractor for increasing number of inducing points ($m$).
      As the number of inducing points increases, both the NLL and accuracy remain constant with no statistically significant difference.
      This shows that is not necessary to have a large number of inducing points to obtain strong performance.
    }
    \label{table:inducing_ablation}
    \begin{tabular}{lcc}
      \toprule
      $m$  & Accuracy (\%) $\uparrow$ & NLL    $\downarrow$ \\ \hline
      10   & 95.56$\pm$0.04           & 0.187$\pm$0.004     \\
      50   & 95.54$\pm$0.06           & 0.182$\pm$0.002     \\
      100  & 95.35$\pm$0.06           & 0.183$\pm$0.002     \\
      1000 & 95.49$\pm$0.05           & 0.180$\pm$0.001     \\
      \bottomrule
    \end{tabular}
  \end{minipage}\hfill
  \begin{minipage}{0.45\linewidth}
    \centering
    \caption{ECE calibration results on CIFAR-10 with 15 bins and no post-hoc scaling. DUQ and SNGP results obtained from \citet{liu2020simple}.}
    \label{table:calibration}
    \begin{tabular}{ll}
      \toprule
      Method & ECE                 \\ \hline
      DUE    & 0.01795  +/- 0.0015 \\
      DUQ    & 0.034 +/- 0.002     \\
      SNGP   & 0.018 +/- 0.001     \\
      \bottomrule
    \end{tabular}
  \end{minipage}
\end{table}

\subsection{Sparse GP Approximation and Variational Inference}
Exact inference for the classification likelihood is not tractable because the softmax function is not a conjugate likelihood to the GP prior.
Additionally, while exact inference is possible for the regression case, the computational complexity scales cubically with the number of data points, thus it is not suitable for large datasets.
Thus, for both regression and classification we use a sparse GP approximation and variational inference.

We use the sparse GP approximation of \citet{titsias2009variational} which augments the model with $M$ inducing inputs, $Z \in \R^{M \times J}$, where $J$ is the dimensionality of the feature space.
The associated inducing variables, $U \in \R^{M \times T}$, give the function value at each inducing input.
Together, the inducing inputs and inducing variables approximate the full dataset.
To perform inference in this model we use the variational approximation introduced by \citet{hensman2015scalable}.
Here $Z$ are treated as variational parameters.
$U$ are random variables with prior $p(U) = \prod_{t=1}^T \Normal(U_{[:,t]} \given \mu_t(Z), \bar{k}(Z,Z))$, and variational posterior $q(U) = \prod_{t=1}^T \Normal(U_{[:,t]} \given \bm{m}_t, S_t)$, where $\bm{m}_t \in \R^{M}$ and $S_t \in \R^{M \times M}$ are variational parameters and initialized at the zero vector and the identity matrix respectively.
The approximate predictive posterior distribution at test points $X^*$ is then
\begin{equation}
  q(F^* \given Y; X, X^*) = \int p(Y \given F^*) p(F^* \given U; X^*, Z) \prod_{t=1}^T q(U_{[:,t]} \given \bm{m}_t, S_t) \diff U \diff F^*.
\end{equation}
Here $p(F^* \given U; X^*, Z)$ is a Gaussian distribution for which we have an analytic expression, see \citet{hensman2015scalable} for details.
Note that we deviate from \citet{hensman2015scalable} in that our input points $x$ are mapped into feature space just before computing the base kernel, while inducing points are used as is (they are defined in feature space).

The variational parameters $Z$, $\bm{m}_t$, and $S_t$, alongside the feature extractor parameters $\theta$ and model hyparparameters $\mu_t$, $l_t$, and $\sigma^2$, are all learned by maximizing a lower bound on the log marginal likelihood, known as the ELBO, $\mathcal{L}$.
For the variational approximation above, this is defined as
\begin{equation}
  \label{eq:elbo}
  p(Y; X) \geq
  \mathcal{L} = \sum_{i=1}^N
  \E_{q(F_{[i,:]} \given \bm{m}_1, \ldots, \bm{m}_T, S_1, \ldots, S_T; \bm{x}_i, Z)}
  \left[ \log p(Y_{[i,:]} \given F_{[i,:]}) \right]
  - \KL(q(U) || p(U)).
\end{equation}
Both terms can be computed analytically when the likelihood is Gaussian and for classification (i.e. a non Gaussian likelihood) we do MC sampling.
Armed with this objective function, we can learn the model parameters and hyperparameters, and variational parameters, using stochastic gradient descent.
To accelerate optimization we additionally use the whitening procedure of \citet{matthews2017scalable}.
We specify the precise optimizer configuration for each experiment later in this section.

\subsection{Making Predictions and Measuring Uncertainty}
For regression tasks we directly use the function values $F^*$ above as the predictions.
We use the mean of $q(F^* \given Y; X, X^*)$ as the prediction, and the variance as the uncertainty.

For classification tasks we need the posterior over the class probabilities, $q(Y^* \given Y; X, X^*)$, rather than the latent function values.
Thus, we approximate the integral
\begin{equation}
  \bar{Y}^* = \int \operatorname{softmax}(F^*) q(F^* \given Y; X, X^*) \diff F^*,
\end{equation}
using Monte Carlo samples (32 in practice), which are very fast to compute and do not require additional forward passes.
Note that we consider the inputs \emph{independent} at test time (i.e. we only take the diagonal of the posterior covariance) which is especially important when detecting out of distribution data.

The predicted class for input $i$ is then the most likely class in $\bar{Y}^*_{[i,:]}$.
To estimate the uncertainty of the prediction we compute the entropy of $\bar{Y}^*_{[i,:]}$:
\begin{equation}
  \operatorname{entropy}(\bar{Y}^*_{[i,:]}) = - \sum_c \bar{Y}^*_{[i,c]} \log \bar{Y}^*_{[i,c]} .
\end{equation}

\subsection{Complexity Penalty and Feature Collapse}
\label{appendix:complexity_penalty}

In this section we show that the DKL objective can lead to feature collapse for both exact and inducing point GPs.
In \citet{ober2021promises}, it was shown that the complexity penalty (Equation \ref{eq:marginal_likelihood}) can lead to overfitting in DKL.
We extend upon this result by linking the penalty to feature collapse.
Additionally, we provide evidence for the fact that this pathology also occurs in inducing point GPs.

\begin{prop}
  \label{prop:feature_collapse}
  The marginal likelihood of a GP with a neural feature extractor, i.e. with a kernel function $k(f_\theta(\cdot), f_\theta(\cdot))$ (DKL) where $f$ is a deep neural network parameterised by $\theta$, can be made arbitrarily large if the feature extractor $f_\theta$ is allowed to map data points arbitrarily close together.
\end{prop}
\begin{proof}[Informal Proof of Proposition \ref{prop:feature_collapse}]

  The marginal likelihood for the DKL GP can be written in the following form, with $K = K(f_\theta(X), f_\theta(X))$
  \begin{align}
    \label{eq:marginal_likelihood}
    \log p(\mathbf{y}) & =\log \mathcal{N}\left(\mathbf{y} \mid \mathbf{0}, \sigma_f K+\sigma_{n}^{2} I_{N}\right)                                                                                                                                                    \\
                       & =-\underbrace{\frac{1}{2} \log \left|\sigma_f K+\sigma_{n}^{2} I_{N}\right|}_{\text {(a) complexity }}-\underbrace{\frac{1}{2} \mathbf{y}^{T}\left(\sigma_f K+\sigma_{n}^{2} I_{N}\right)^{-1} \mathbf{y}}_{\text {(b) data fit }} \nonumber
  \end{align}

  where $\sigma_f, \sigma_n$ are learnable scales for the kernel and observation noise respectively.
  Firstly it can be shown that

  \begin{lem}
    \label{lem:ober}
    Consider the GP marginal likelihood in Equation \ref{eq:marginal_likelihood}.
    The ``data fit'' term will attain the value $-\frac{N}{2}$ (where N is the number of data points) at the optimum of the marginal likelihood. \citep{ober2021promises}
  \end{lem}
  \begin{proof}[Proof of Lemma \ref{lem:ober}]
    A full proof is given in appendix A of \citet{ober2021promises}.
    A key part of this proof is to re-parameterize the noise term $\sigma_n$ as $\hat\sigma_n = \sigma_n / \sigma_f$, in order to force $\sigma_f$ out as a factor in the data fit term.
    The proof then follows by taking derivatives with respect to $\sigma_f$ and noting this derivative must be zero at optimality.
  \end{proof}

  The new parameter $\sigma_n$ is then independent of the kernel length scale.
  Using this parameterization, we can write the data complexity term as
  \begin{align}
    \frac{1}{2} \log \left|\sigma_f K+\sigma_{n}^{2} I_{N}\right|=\frac{N}{2} \log \sigma_{f}^{2}+\frac{1}{2} \log \left|K+\hat{\sigma}_{n}^{2} I_{N}\right|
  \end{align}

  where $K$ is the kernel matrix on the data. Note that the kernel function, scale of the kernel $\sigma_f$ and data noise are all learnable, and  chosen to maximize the marginal likelihood.
  If we assume that $\sigma_f$ is at an optimum, as in \citet{ober2021promises}, then the complexity penalty can only be minimized through the $ \log |K + \hat{\sigma}_n^2 I|$ term.

  If we have a parameterized kernel such as a neural network, K becomes $K(f_\theta(x), f_\theta(x))$ and the deep model can map the inputs to arbitrary locations.
  In particular, note that the determinant of the Gram matrix $K$ is zero if and only if two rows of the matrix are co-linear, which occurs if the feature maps of two different data points $f_\theta(x_i), f_\theta(x_j)$ are the same.
  Additionally, the noise parameter $\sigma_n$ of the model is a learnable parameter, which can take different values based on our assumption about the data.
  By moving the feature representations of two data points as close to each other as possible and learning $\sigma_n$ to be close to zero (under the assumption of no observation noise), the log determinant can become arbitrarily small: $|K + \hat{\sigma}_n^2 I_n| \to 0$, which will make the complexity term tend to negative infinity, thus increasing the marginal likelihood without bound.
\end{proof}

The cause of the pathology here is that the ``neural'' part of the kernel allows the model to move the data points around - whereas in a standard GP the data point locations are fixed, rather than being explicitly optimized.
The limited capacity of the network prevent the likelihood going to infinity in practice, but the objective is clearly pathological without constraints on $f_\theta$.

This result is for an exact GP, but we know from \citet{burt2019rates} that the ELBO becomes tight with sufficient inducing points, therefore the above analysis applies to the approximate GP setting of DUE as well, since it is a property of the GP marginal likelihood that the variational bound is approximating.

\subsection{Implementation}
\label{appendix:compute}
The inducing points locations are initialized using centroids obtained from performing k-means on the feature representation of 1,000 points randomly chosen from the training set.
The initial length scales are computed by taking the average pairwise euclidean distance between feature representation of the 1,000 points.
The models are implemented using GPyTorch \citep{gardner2018gpytorch} and PyTorch \citep{paszke2019pytorch} and use their default values if not otherwise specified.

\section{EXPERIMENTAL DETAILS}
\label{appendix:experimental_details}

\subsection{1D and 2D Experiments}
We perform these experiments using a simple feed-forward ResNet, similar to \citet{liu2020simple}.
The first linear layer maps from the input to the initial feature representation and does not have an activation function.
After which the model is a fully connected ResNet consisting of blocks computing $x' = x + f(x)$, with $f(\cdot)$ a combination of a linear mapping and a ReLU activation function.
Spectral normalization is applied to the linear mapping for which we use the implementation of \citet{behrmann2019invertible}.
We use the Adam optimizer for regression and SGD for the two moons classification with learning rate 0.01.
We use 4 layers with 128 features, a Lipschitz constant of 0.95, a single power iteration, an RBF kernel and additive Gaussian noise with standard deviation 0.1.
For toy regression, we use 20 inducing points and for two moons we use four.
For the Deep Ensemble in Figure \ref{fig:rff_appendix} we train 10 separate fully connected ResNet models, using a different initialization and data order for each and train to maximize the Gaussian log-likelihood as described in Section 2.4 of \citet{lakshminarayanan2017simple}.

\subsection{CIFAR-10}
For the WRN, we follow the experimental setup and implementation of \citet{zagoruyko2016wide}.
This means that for CIFAR-10, we use depth 28 with widen factor 10 and BasicBlocks with dropout.
We train for the prescribed 200 epochs (no early stopping) with batch size 128, starting with learning rate 0.1 and dropping with a factor of 0.2 at epoch 60, 120 and 160.
We use momentum 0.9 and weight decay 5e-4.
We select 20\% of the training data at random as a validation set for hyper-parameter tuning, but obtain the final model by using the full training set with the final set of hyper-parameters.
SV-DKL was trained using the suggested values in the GPyTorch \citep{gardner2018gpytorch} tutorial: grid size set to 64 and grid bounds between -10 and 10.

For spectral normalization, we use the implementation of \citet{behrmann2019invertible}, and also constrain batch normalization as described in \citep{gouk2018regularisation}.
We use 1 power iteration and use the lowest Lipschitz constant that still allows for good accuracy, which we found to be around 3 in practice.
We set the momentum of spectral batch normalization to 0.99 to reduce the variance of the running average estimator of the empirical feature variance, which can be a source of instability.

In the out-of-distribution detection experiment, it is crucial to preprocess the SVHN images in the same way as CIFAR-10, as we cannot know at test time from which dataset the image comes and the only sensible procedure is to preprocess as if it was coming from the training dataset.

\begin{figure}[t]
  \begin{subfigure}{0.32\linewidth}
    \centering
    \includegraphics[width=\linewidth]{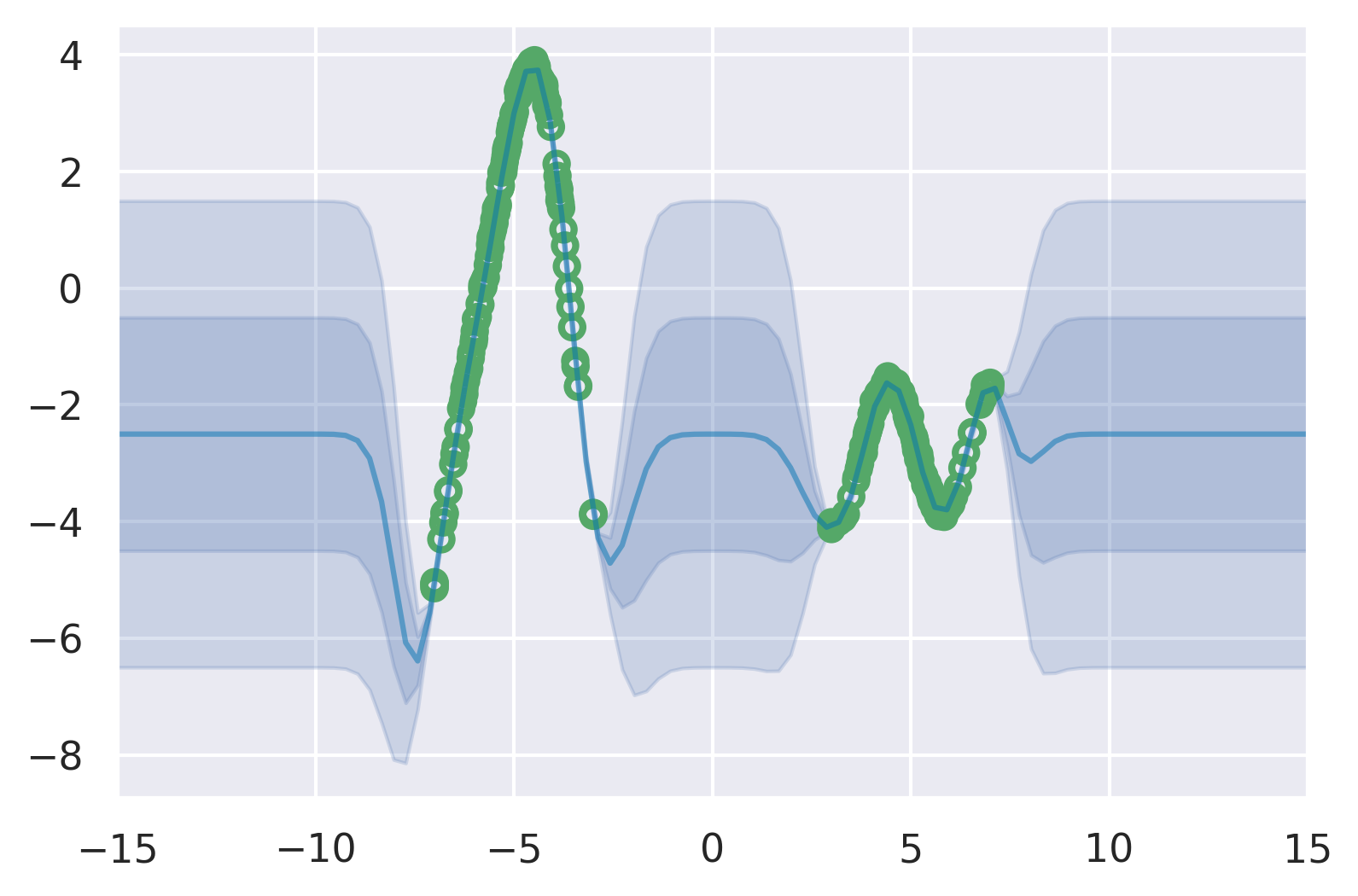}
    \caption{exact GP - 1K}
    \label{fig:exact_1k}
  \end{subfigure}%
  \begin{subfigure}{0.32\linewidth}
    \centering
    \includegraphics[width=\linewidth]{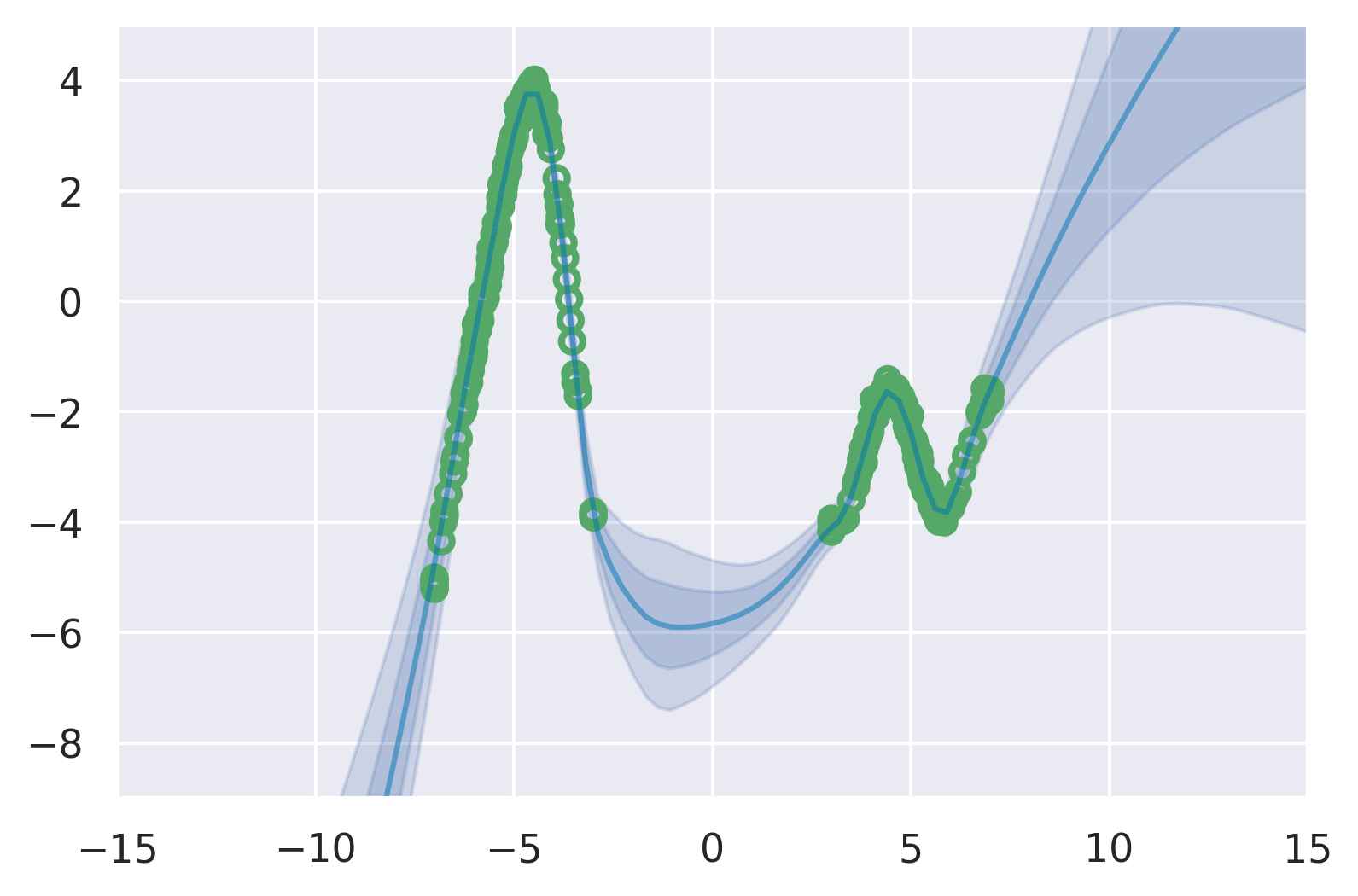}
    \caption{Deep Ensemble - 1K}
    \label{fig:deepensembles_1k}
  \end{subfigure}%
  \begin{subfigure}{0.32\linewidth}
    \centering
    \includegraphics[width=\linewidth]{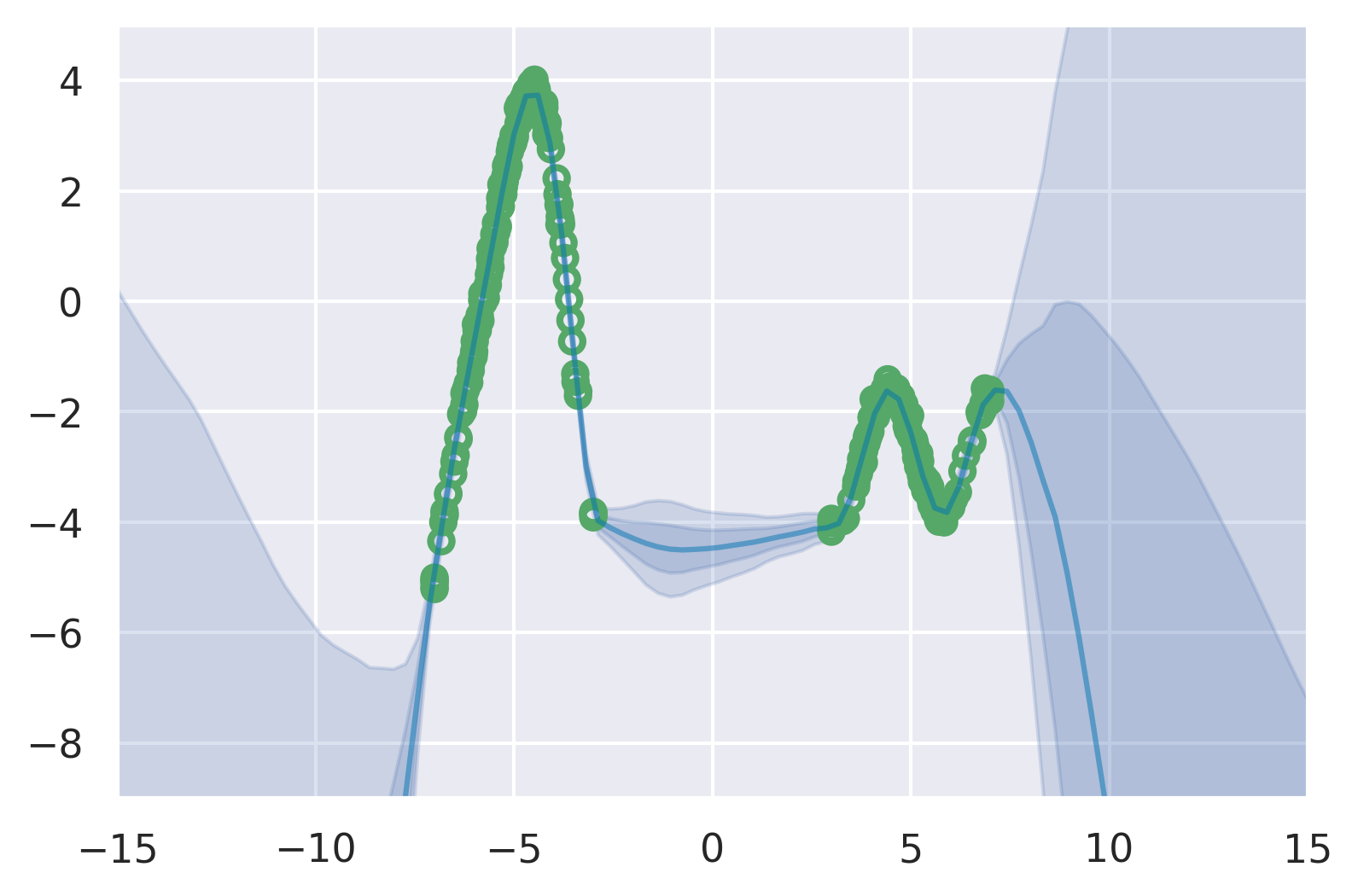}
    \caption{Deep Ensemble - 1M}
    \label{fig:deepensembles_1m}
  \end{subfigure}%
  \caption{
    In green 300 example training data points and in blue the prediction including uncertainty (one and two std).
    The exact GP could only be fit on 1K data points, so the 1M version is omitted.
    The Deep Ensemble does not revert to any prior, but also does not concentrate fully in the 1M case.
  }
  \label{fig:rff_appendix}
\end{figure}

\subsection{Regression Experiment}
\label{appendix:causal_experiments}
Following \citet{shalit2017estimating}, we use 63\%/27\%/10\% train / validation / test splits and report the RMSE evaluated on the test set over 1000 trials.
For each trial, we train for a maximum of 750 epochs with batch size 100 and report results on the model with the lowest negative log-likelihood evaluated on the validation set.
We employ Adam optimization with a learning rate of 0.001 and batch size of 100.

The feature extractor uses a feed-forward ResNet architecture with 3 layers, 200 hidden units per layer, and ELU activations.
Dropout is applied after each activation at a rate of 0.1.
The feature extractor takes the individual $x$ as input, and treatment $t$ is appended to the output of the feature extractor, in similar fashion to the TARNet architecture.
Spectral normalization with value 0.95 is used on all layers of the feature extractor.
The output GP uses a Mat\'ern kernel with $\nu=\frac32$ and 100 inducing points.

The baselines in Table \ref{table:causal} are: BART which is based on decision trees \citep{chipman2010bart}, BCEVAE which is the MC Dropout extension of CEVAE \citep{louizos2017causal}, and DKLITE \citep{zhang2020learning} which is a recent method also based on DKL.
Results are obtained from \citet{jesson2020identifying}.
DKLITE works by jointly training a VAE \citep{kingma2013auto} with a GP on the latent dimensions of the VAE.
For the DKLITE experiments we use the open source code from the authors\footnote{\url{https://bitbucket.org/mvdschaar/mlforhealthlabpub/src/master/alg/dklite/}}.
We write a custom loop over the IHDP dataset to follow the above protocol.
BTARNet, ensemble of TARNet and DUE all use the same feature extractor, as detailed in Appendix \ref{appendix:experimental_details}.
We report the root mean squared error between the predicted CATE and the true CATE (which is given in the dataset) to assess the error on the held out data set (lower is better).

\end{document}